\documentclass[11pt]{article}
\pdfoutput=1
\usepackage{enumerate}
\usepackage{pdfsync}
\usepackage[OT1]{fontenc}
\usepackage[usenames]{color}

\usepackage{smile}
\usepackage[colorlinks,
            linkcolor=red,
            anchorcolor=blue,
            citecolor=blue
            ]{hyperref}
\usepackage{mathrsfs}
\usepackage{fullpage}
\usepackage{hyperref}
\usepackage[protrusion=true,expansion=true]{microtype}
\usepackage{float}
\usepackage{subfigure}
\usepackage{setspace}
\usepackage{algorithm}
\usepackage{algorithmic}
\usepackage{balance}
\usepackage{enumitem}

\newtheorem{condition}[theorem]{Condition}
\usepackage{color}

\makeatletter
\newcommand*{\rom}[1]{\expandafter\@slowromancap\romannumeral #1@}
\makeatother

\begin{document}
\title{\huge A Knowledge Transfer Framework for Differentially Private Sparse Learning}
\author
{
 Lingxiao Wang\thanks{Department of Computer Science, University of California, Los Angeles, CA 90095, USA; e-mail: {\tt 
			lingxw@cs.ucla.edu}} ~~and~~
 Quanquan Gu\thanks{Department of Computer Science, University of California, Los Angeles, CA 90095, USA; e-mail: {\tt
 		~~~~qgu@cs.ucla.edu}}
}
\date{}
\maketitle

\begin{abstract} 
 We study the problem of estimating high dimensional models with underlying sparse structures while preserving the privacy of each training example. We develop a differentially private high-dimensional sparse learning framework using the idea of knowledge transfer. More specifically, we propose to distill the knowledge from a ``teacher'' estimator trained on a private dataset, by creating a new dataset from auxiliary features, and then train a differentially private ``student'' estimator using this new dataset. In addition, we establish the linear convergence rate as well as the utility guarantee for our proposed method. For sparse linear regression and sparse logistic regression, our method achieves improved utility guarantees compared with the best known results \citep{kifer2012private,wang2019differential}. We further demonstrate the superiority of our framework through both synthetic and real-world data experiments.
\end{abstract}

\section{Introduction}\label{sec:intro}
In the Big Data era, sensitive data such as genomic data and purchase history data, are ubiquitous, which necessitates learning algorithms that can protect the privacy of each individual data record.  A rigorous and standard notion for privacy guarantees is differential privacy \citep{dwork2006calibrating}. By adding random noise to the model parameters (output perturbation), some intermediate steps of the learning algorithm (gradient perturbation), or the objective function of learning algorithms (objective perturbation), differentially private algorithms ensure that the trained models can learn the statistical information of the population without leaking any information about the individuals.  In the last decade, a surge of differentially private learning algorithms \citep{chaudhuri2009privacy,chaudhuri2011differentially,kifer2012private,bassily2014differentially,talwar2015nearly,zhang2017efficient,wang2017differentially,wang2018empirical,jayaraman2018distributed} for empirical risk minimization have been developed. However, most of these studies only consider the classical setting, where the problem dimension is fixed. In the modern high-dimensional setting where the problem dimension can increase with the number of observations, all these empirical risk minimization algorithms fail. A common and effective approach to address these issues is to assume the model has a certain structure such as sparse structure or low-rank structure. In this paper, we consider high-dimensional models with sparse structure. Given a dataset $S=\{(\xb_i,y_i)\}_{i=1}^n$, where $\xb_i \in \RR^d$ and $y_i\in \RR$ are the input vector and response of the $i$-th example, our goal is to estimate the underlying sparse parameter vector $\btheta^*\in\RR^d$, which has $s^*$ nonzero entries, by solving the following $\ell_2$-norm regularized optimization problem with the sparsity constraint
\begin{align}\label{eq:structure_opt}
    \min_{\btheta\in\RR^d} \bar L_S(\btheta):=L_S(\btheta)+\frac{\lambda}{2}\|\btheta\|_2^2\quad \text{subject to}\quad \|\btheta\|_0\leq s,
\end{align}
where $L_S(\btheta):=n^{-1}\sum_{i=1}^n \ell(\btheta;\xb_i,y_i)$ is the empirical loss on the training data, $\ell(\btheta;\xb_i,y_i)$ is the loss function defined on the training example $(\xb_i,y_i)$, $\lambda\geq 0$ is a regularization parameter, $\|\btheta\|_0$ counts the number of nonzero entries in $\btheta$, and $s$ controls the sparsity of $\btheta$. The reason we add an extra $\ell_2$ regularizer to \eqref{eq:structure_opt} is to ensure the strong convexity of the objective function without making any assumption on the data. %and ensure the privacy guarantee independent of any assumption on the data.
% Let $\btheta^{\mathrm{p}}$ be the  output of a differentially private learning algorithm, we use $ L(\btheta^{\mathrm{p}})-L(\btheta^*)$ to denote the utility guarantee of such algorithm.

In order to achieve differential privacy for sparse learning, a line of research \citep{kifer2012private,thakurta2013differentially,jain2014near,talwar2015nearly,wang2019differential} studied differentially private learning problems in the high-dimensional setting, where the problem dimension can be larger than the number of observations. For example, \citet{jain2014near} provided a differentially private algorithm with the dimension independent utility guarantee. However, their approach only considers the case when the underlying parameter lies in a simplex. %and
%the corresponding utility guarantee is sub-optimal in terms of other parameters. 
For sparse linear regression, \citet{kifer2012private,thakurta2013differentially} proposed a two-stage approach to ensure differentially privacy. In detail, they first estimate the support set of the sparse model parameter vector using some differentially private model selection algorithm, and then estimate the parameter vector with its support restricted to the estimated subset using the objective perturbation approach \citep{chaudhuri2009privacy}. Nevertheless, %their algorithm requires the exact minimizer to establish both differential privacy and utility guarantees. Furthermore, 
the support selection algorithm, like exponential mechanism, is computational inefficient or even intractable in practice. %Although their method does not require the extra support selection procedure, the corresponding utility guarantee depends on the $\ell_2$-norm bound of the input data.
\citet{talwar2015nearly} proposed a differentially private algorithm for sparse linear regression by combining the Frank-Wolfe method \citep{frank1956algorithm} and the exponential mechanism. Although their utility guarantee is worse than \citet{kifer2012private,wang2019differential}, it does not depend on the restricted strong convexity (RSC) and smoothness (RSS) conditions \citep{negahban2009unified}. Recently, \citet{wang2019differential} developed a differentially private iterative gradient hard thresholding (IGHT) \citep{jain2014iterative,yuan2014gradient} based framework for sparse learning problems by injecting Gaussian noise into the intermediate gradients. 
% Different from aforementioned methods \citep{kifer2012private,wang2019differential}, our proposed knowledge transfer framework can significantly improve the utility guarantee by getting rid of the dependence on the $\ell_2$-norm bound of the input vectors. Without the RSC and RSS conditions, our method has an improved utility guarantee compared with \citet{talwar2015nearly}, which demonstrates the advantage of our knowledge transfer based framework. Table \ref{table:sparse} summarizes the detailed comparisons of  different methods for sparse linear regression. 
However, all the aforementioned methods either have unsatisfactory utility guarantees or are computationally inefficient. For example, the utility guarantees provided by \citet{kifer2012private,thakurta2013differentially,wang2019differential} depend on the $\ell_2$-norm bound of the input vector, which can be in the order of $O(\sqrt{d})$ and grows as $d$ increases in the worse case.  While the utility guarantee of the algorithm proposed by \citet{talwar2015nearly} only depends on the $\ell_\infty$-norm bound of the input vector, it has a worse utility guarantee, and its convergence rate is sub-linear. %which suffers from the slow convergence rate for large-scale problems.
\begin{figure}%
% \begin{wrapfigure}{r}{0.5\textwidth}
%\begin{center}
  %\vspace{.3in}
%% label for first subfigure
\centering
    \includegraphics[width=0.7\textwidth]{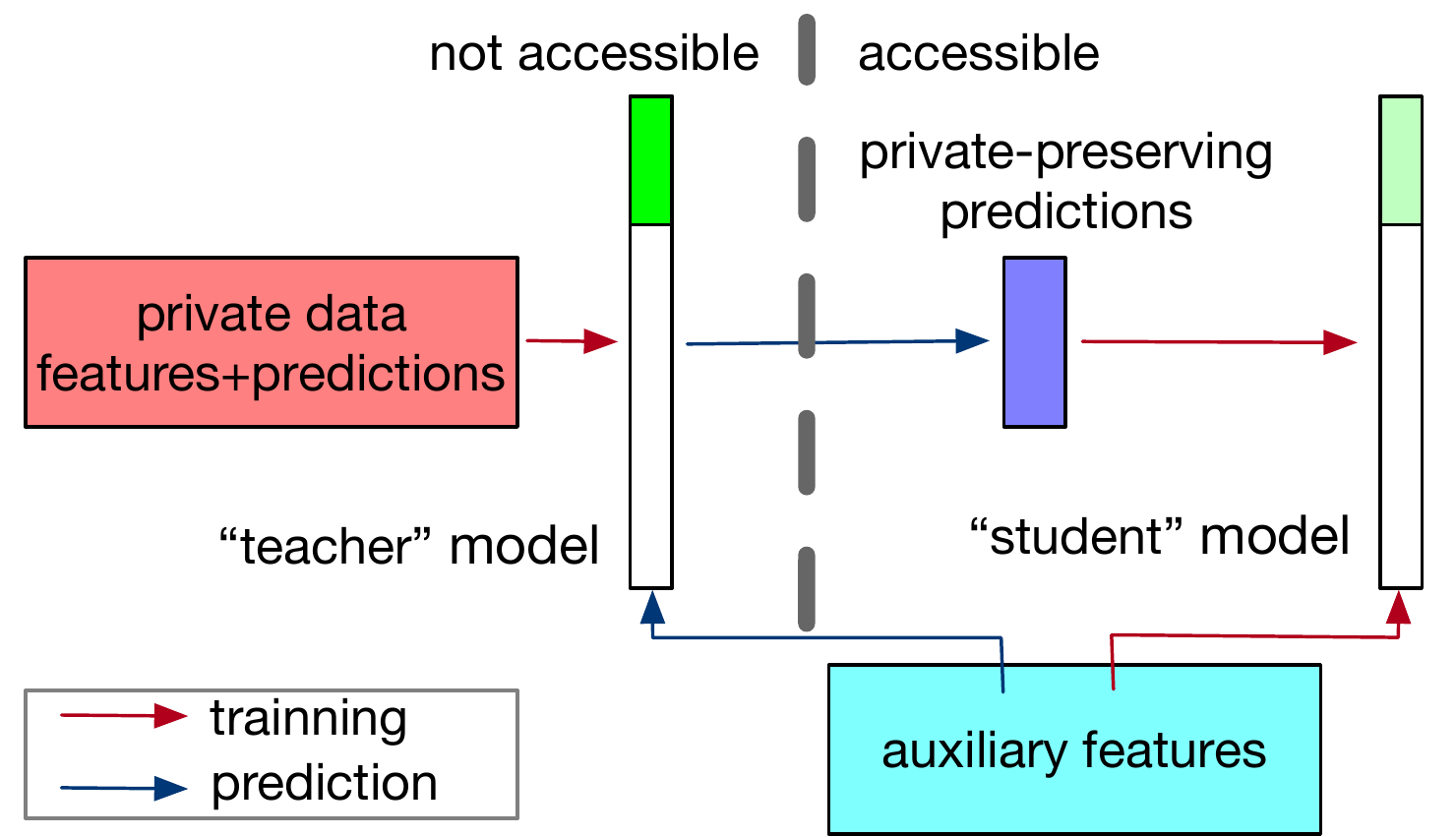}%% 
    %\end{center}
  %\vspace{.3in}
  \caption{Illustration of the proposed framework: (1). A ``teacher'' estimator is trained using the private dataset; (2). A new private-preserving dataset is generated using the auxiliary features and their private predictions output by the  ``teacher'' estimator; (3). A differentially private ``student'' estimator is trained using the newly generated  dataset.}\label{fig:framework}%% label for entire figure
   %\end{wrapfigure}
\end{figure}
%  Therefore, is it possible for us to achieve both the fast convergence rate and the strong utility guarantee for the high-dimensional sparse learning problem? 

Therefore, a natural question is whether we can achieve the best of both worlds: a strong utility guarantee and high computational efficiency. %the biggest challenge in the current differentially private sparse leaning problem is how to achieve both the fast convergence rate and the strong utility guarantee. 
To this end, we propose to make use of the idea of knowledge distillation \citep{bucilua2006model,hinton2015distilling}, which is a knowledge transfer technique originally introduced as a mean of model compression. The original motivation of using knowledge distillation is to use a large and complex ``teacher'' model to train a small ``student'' model, while maintaining its accuracy. For the differentially private sparse learning problem, similar idea can be applied here: we can use a non-private ``teacher'' model to train a differentially private ``student'' model, while preserving the sparse information of the ``teacher'' model. We notice that several knowledge transfer approaches have been recently investigated in the differentially private classification problem \citep{hamm2016learning,papernot2016semi,bassily2018model,yoon2018pate}. Nevertheless, the application of knowledge distillation to the generic differentially private high-dimensional sparse learning problem is new and has never been studied before.

%  which is inspired by a series of recent work  with the focus on the differentially private classification problem.
 
In this paper, we propose a knowledge transfer framework for solving the high-dimensional sparse learning problem on a private dataset, which is illustrated in Figure \ref{fig:framework}. Our proposed algorithm is not only very efficient but also has improved utility guarantees compared with the state-of-the-art methods.  More specifically, we first train a non-private ``teacher'' model using IGHT % Iterative Gradient Hard Thresholding (IGHT) \cite{yuan2014gradient,jain2014iterative} 
from the private dataset. Based on this ``teacher'' model, we then construct a privacy-preserving dataset using some auxiliary inputs, which are drawn from some given distributions or public datasets. Finally, by training a ``student'' model using IGHT again based on the newly generated dataset, we can obtain a differentially private sparse estimator.  
Table \ref{table:sparse} summarizes the detailed comparisons of  different methods for sparse linear regression, and we summarize the contributions of our work as follows
\begin{itemize}[leftmargin=*]
    \item Our proposed differentially private framework can be applied to any smooth loss function, which covers a broad family of sparse learning problems. In particular, we showcase the application of our framework to sparse linear regression and sparse logistic regression.
    % \item We prove a sharp utility guarantee for our proposed framework. Specifically, for sparse linear regression and sparse logistic regression, our method achieves $O\big(K^2s^{*2}\sqrt{\log d}/(n\epsilon)\big)$ and $O\big(K\sqrt{s^{*3}\log d}/(n\epsilon)\big)$ utility guarantees respectively, where $n$ is the number of observations, $s^*$ is the sparse parameter of the underlying true parameter, $d$ is the dimension, $K$ is the $\ell_\infty$-norm bound of the input vectors, and $\epsilon$ is the privacy budget. Compared with the best known results \cite{kifer2012private,wang2019differential}, our utility guarantees only depends on the $\ell_\infty$-norm bound of the input vectors instead of the $\ell_2$-norm bound.
     \item We prove a better utility guarantee and establish a liner convergence rate for our proposed method. For example, for sparse linear regression, our method achieves $O\big(K^2s^{*2}\sqrt{\log d}/(n\epsilon)\big)$ utility guarantee, where %$n$ is the number of observations, $s^*$ is the sparse parameter of the underlying true parameter vector, $d$ is the dimension,
     $K$ is the $\ell_\infty$-norm bound of the input vectors, and $\epsilon$ is the privacy budget. Compared with the best known utility bound $O\big(\tilde K^2s^{*2}\log d/(n^2\epsilon^2)\big)$ \citep{kifer2012private,wang2019differential} ( $\tilde K$ is the $\ell_2$-norm bound of the input vectors), our utility guarantee is better than it by a factor of $O\big(\tilde K^2\sqrt{\log d}/(K^2n\epsilon)\big)$. Considering that $\tilde K$ can be $\sqrt{d}$ times larger than $K$, the improvement factor can be as large as $O\big(d\sqrt{\log d}/(n\epsilon)\big)$. Similar improvement is achieved for sparse logistic regression.
     
     \item With the extra sparse eigenvalue condition \citep{bickel2009simultaneous} on the private data, our method can achieve $O\big(K^2s^{*3}\log d/(n^2\epsilon^2)\big)$ utility guarantee for sparse linear regression. It is better than the best known result \citep{kifer2012private,wang2019differential} $O\big(\tilde K^2s^{*2}\log d/(n^2\epsilon^2)\big)$ by a factor of $O\big(\tilde K^2/(K^2s^*)\big)$, which can be as large as $O\big(d/s^*\big)$. Similar improvement is also achieved for sparse logistic regression.
    
    % and our framework do not require any extra support selection procedure.
  
    % has better utility guarantee than the state-of-the art method \cite{hardt2012beating,hardt2013beyond} in the same setting of differential privacy.
    % \item We establish a linear convergence rate for our algorithm, which is faster than existing Frank-Wolfe algorithm based method \citep{talwar2015nearly} with a sub-linear convergence rate. Compared with \citet{kifer2012private}, our algorithm does not require any prior support selection algorithm, like exponential mechanism, which is computationally inefficient or even intractable.
\end{itemize}
\begin{table*}[!t]
\caption{Comparison of different algorithms for sparse linear regression in the setting of $(\epsilon,\delta)$-DP. We report the utility bound achieved by the privacy-preserving mechanisms, and ignore the $\log (1/\delta)$ term. Note that $n\epsilon\gg 1$, $\xb_i$ denotes the $i$-th input vector, and $\upsilon$ is the probability that the support selection procedure can successfully recover the true support.}
\begin{center}
\small
	\begin{tabular}{cccccc}
		\toprule
		\multirow{2}{*}{Algorithm} &\multirow{2}{*}{Data Assumption} & \multirow{2}{*}{Utility}&Convergence& \multirow{2}{*}{Utility Assumption}\\&&&Rate&\\
		\midrule
		Frank-Wolfe&\multirow{2}{*}{$\max_{i\in[n]}\|\xb_i\|_{\infty}\leq 1$}&\multirow{2}{*}{$O\Big(\frac{\log (nd)}{(n\epsilon)^{2/3}}\Big)$} &\multirow{2}{*}{Sub-linear}&\multirow{2}{*}{No}\\\citep{talwar2015nearly}&&&&&\\
		\midrule
		Two Stage &\multirow{2}{*}{$\max_{i\in[n]}\|\xb_i\|_2\leq \tilde K$}& \multirow{2}{*}{$O\Big(\frac{\tilde K^2s^{*2}\log (2/\upsilon)}{(n\epsilon)^{2}}\Big)$} &\multirow{2}{*}{ NA} &\multirow{2}{*}{RSC/RSS}\\\citep{kifer2012private}&&&&&\\
		\midrule
				DP-IGHT &\multirow{2}{*}{$\max_{i\in[n]}\|\xb_i\|_2\leq \tilde K$}& \multirow{2}{*}{$O\Big(\frac{\tilde K^2s^{*2}\log d}{(n\epsilon)^{2}}\Big)$} &\multirow{2}{*}{ Linear} &\multirow{2}{*}{RSC/RSS}\\\citep{wang2019differential}&&&&&\\
					\midrule
		\textbf{DPSL-KT} &\multirow{2}{*}{$\max_{i\in[n]}\|\xb_i\|_{\infty}\leq K$}& \multirow{2}{*}{$O\Big(\frac{K^2s^{*2}\sqrt{\log d}}{n\epsilon}\Big)$}&\multirow{2}{*}{Linear}&\multirow{2}{*}{No}\\$\lambda>0$&&&&&\\
		\midrule
		\textbf{DPSL-KT} &$\max_{i\in[n]}\|\xb_i\|_{\infty}\leq K$& \multirow{2}{*}{$O\Big(\frac{K^2s^{*3}\log d}{(n\epsilon)^{2}}\Big)$}&\multirow{2}{*}{Linear}&\multirow{2}{*}{RSC/RSS}\\$\lambda=0$&RSC/RSS&&&&\\
		\bottomrule
	\end{tabular}
\label{table:sparse}
\end{center}
\end{table*}
% The remaining paper is organized as follows: We discuss related work in Section~\ref{sec:2} and provide background on differential privacy in Section~\ref{sec:3}. In Section~\ref{sec:4}, we introduce several examples that can be covered by our proposed framework.  We describe our proposed differentially private framework in Section~\ref{sec:5}. Section~\ref{sec:6} presents the main results of our method.
% Section~\ref{sec:7} reports on results from experiments. We conclude the paper and point out future work in Section \ref{sec:conclusions}.

\noindent \textbf{Notation}.
 For a $d$-dimensional vector $\xb = [x_1,...,x_d]^{\top}$, we use $\|\xb\|_{2} = (\sum_{i=1}^{d}|x_{i}|^{2})^{1/2}$ to denote its $\ell_2$-norm, and use $\|\xb\|_{\infty}=\max_i|x_i|$ to denote its $\ell_\infty$-norm.  We let $\supp(\xb)$ be the index set of nonzero entries of $\xb$, and $\supp(\xb,s)$ be the index set of the top $s$ entries of $\xb$
in terms of magnitude. We use $\cS^n$ to denote the input space with $n$ examples and $\cR,\cR^\prime$ to denote the output space. Given two sequences $\{a_n\},\{b_n\}$, if there exists a constant $0<C<\infty$ such that $a_n\leq Cb_n$, we write $a_n = O(b_n)$, and  we use $\tilde O(\cdot)$ to hide the logarithmic factors. We use $\Ib_d\in\RR^{d\times d}$ to denote the identity matrix. Throughout the paper, we use $\ell_i(\cdot)$ as the shorthand notation for $\ell(\cdot;\xb_i,y_i)$, and $\btheta_{\min}$ to denote the minimizer of problem \eqref{eq:structure_opt}.
% We use $[n]$ to denote the index number set from $1$ to $n$.

% %\begin{figure*}[!ht]%
% \begin{wrapfigure}{r}{0.5\textwidth}
% %\begin{center}
%   %\vspace{.3in}
% %% label for first subfigure
%     \includegraphics[width=0.7\textwidth]{figure/distill_framework.pdf}%% 
%     %\end{center}
%   %\vspace{.3in}
%   %\caption{Illustration of the proposed differentially private framework for sparse learning: (1). A ``teacher" estimator is trained using the private dataset; (2). A new private-preserving dataset is generated using the auxiliary features and their corresponding private-preserving predictions output by the  ``teacher" estimator; (3). A differentially private ``student" estimator is trained using the newly generated  dataset.}\label{fig:framework}%% label for entire figure
%   \end{wrapfigure}
% %\end{figure*}
\subsection{Additional Related Work}\label{sec:relate}

To further enhance the privacy guarantee for training data, there has emerged
a fresh line of research \citep{hamm2016learning,papernot2016semi,bassily2018model,yoon2018pate} that  studies the knowledge transfer techniques for the differentially private classification problem. More specifically, these methods propose to first train an ensemble of ``teacher'' models based on disjoint subsets of the private dataset, and then train a ``student'' model based on the private aggregation of the ensemble. However, their approaches only work for the classification task, and cannot be directly applied to general sparse learning problems. Moreover, their sub-sample and aggregate framework may not be suitable for the high-dimensional sparse learning problem since each ``teacher'' model is trained on a subset of the private dataset, which makes the ``large $d$, small $n$'' scenario even worse. In contrast to their sub-sample and aggregate based knowledge transfer approach, we propose to use the distillation based method \citep{bucilua2006model,hinton2015distilling}, which is more suitable for the high-dimensional sparse learning problem.
\section{Preliminaries}\label{sec:pre}
In this section, we introduce some background and preliminaries about optimization and differential privacy. We first lay out the formal definitions of strongly convex and smooth functions. %conditions, which will be used to establish our main results. 
\begin{definition}
	A function $f:\RR^d\rightarrow\RR$ is $\lambda$-strongly convex, if for any $\btheta_1,\btheta_2\in\RR^d$, 
	\begin{align*}
f(\btheta_1)- f(\btheta_2)-\la\nabla f(\btheta_2),\btheta_1-\btheta_2\ra\geq \frac{\lambda}{2}\|\btheta_1-\btheta_2\|_2^2.
	\end{align*}
\end{definition}
\begin{definition}
	A function $f:\RR^d\rightarrow\RR$ is $\bar \beta$-smooth, if for any $\btheta_1,\btheta_2\in\RR^d$, 
	\begin{align*}
f(\btheta_1)- f(\btheta_2)-\la\nabla f(\btheta_2),\btheta_1-\btheta_2\ra\leq \frac{\bar \beta}{2}\|\btheta_1-\btheta_2\|_2^2.
	\end{align*}
\end{definition}
Next we present the definition of sub-Gaussian distribution \citep{vershynin2010introduction}.
\begin{definition}\label{def:subgrv}
	%We say $X$ is a sub-Gaussian random variable with parameter $\alpha>0$, if $(\EE|X|^p)^{1/p}\leq \alpha\sqrt{p}$ for all $p\geq 1$. In addition, 
	We say $\bX\in\RR^d$ is a sub-Gaussian random vector with parameter $\alpha>0$, if $(\EE|\ub^\top\bX|^p)^{1/p}\leq \alpha\sqrt{p}$ for all $p\geq 1$ and all unit vector $\ub$ with $\|\ub\|_2=1$.
\end{definition}
We also provide the definition of differential privacy.
\begin{definition}[\citep{dwork2006calibrating}]
A randomized mechanism $\cM:\cS^n\rightarrow\cR$ satisfies $(\epsilon,\delta)$-differential privacy if for any two adjacent datasets $S,S'\in \cS^n$ differing by one example, and any output subset $O\subseteq \cR$, it holds that $    \PP[\cM(S)\in O]\leq e^\epsilon\cdot \PP[\cM(S')\in O]+\delta$, where $\delta\in[0,1)$.
% \begin{align*}
%     \PP[\cM(S)\in O]\leq e^\epsilon\cdot \PP[\cM(S')\in O]+\delta.
% \end{align*}
\end{definition}
% If $\delta=0$, it is called  $\epsilon$-differential privacy.  If $\delta\in(0,1)$, we call it approximate differential privacy since it is a relaxation of pure differential privacy. 
Now we introduce the Gaussian Mechanism \citep{dwork2014algorithmic} to achieve $(\epsilon,\delta)$-DP. We start with the definition of $\ell_2$-sensitivity, which is used to control the variance of the noise in Gaussian mechanism.
\begin{definition}[\citep{dwork2014algorithmic}]
For two adjacent datasets $S,S'\in \cS^n$ differing by one example, the $\ell_2$-sensitivity $\Delta_2(q)$ of a function $q:\cS^n\rightarrow\RR^d$ is defined as 
    $\Delta_2(q)=\sup_{S,S'}\|q(S)-q(S')\|_2$.
\end{definition}

Given the $\ell_2$-sensitivity, we can ensure the differential privacy using Gaussian mechanism.
\begin{lemma}[\citep{dwork2014algorithmic}]\label{lemma:GaussianM}
Given a function $q:\cS^n\rightarrow\RR^d$, the Gaussian Mechanism $\cM=q(S)+\ub$, where $\ub\sim N(0,\sigma^2\Ib_d)$, satisfies $(\epsilon,\delta)$-DP for some $\delta>0$, if $\sigma=\sqrt{2\log(1.25/\delta)}\Delta_2(q)/\epsilon$.
\end{lemma}

The next lemma illustrates that $(\epsilon,\delta)$-DP has the post-processing property, i.e., the composition of a data independent mapping $f$ with an $(\epsilon,\delta)$-DP mechanism $\cM$ also satisfies $(\epsilon,\delta)$-DP.
\begin{lemma}[\citep{dwork2014algorithmic}]\label{lemma:com_post}
Consider a randomized mechanism $\cM:\cS^n\rightarrow\cR$ that is $(\epsilon,\delta)$-DP. Let $f:\cR\rightarrow\cR^\prime$ be an arbitrary randomized mapping. Then $f( \cM):\cS^n\rightarrow\cR^\prime$ is $(\epsilon,\delta)$-DP.
\end{lemma}
% \section{Illustrative examples}\label{sec:exmp}
% \input{examples.tex}
\section{The Proposed Algorithm}\label{sec:alg}
In this section, we present our differentially private sparse learning framework, which is illustrated in Algorithm \ref{alg:DPGST2}. Note that Algorithm \ref{alg:DPGST2} will call IGHT algorithm \citep{yuan2014gradient,jain2014iterative} in Algorithm \ref{alg:IGHT}. IGHT enjoys linear convergence rate and is widely used for sparse learning.
Note that for the sparsity constraint, i.e., $\|\btheta\|_0\leq s$, the hard thresholding operator $\cH_{s}(\btheta)$ is defined as follows: $[\cH_s(\btheta)]_i=\theta_i$ if $i\in \supp(\btheta,s)$ and $[\cH_s(\btheta)]_i=0$ otherwise, for $i \in [d]$. It preserves the largest $s$ entries of $\btheta$ in magnitude. Equipped with IGHT, our framework also has a linear convergence rate for solving high-dimensional sparsity constrained problems.
%\vspace{-0.1in}
\begin{algorithm}[!h]
	\caption{Differentially Private Sparse Learning via Knowledge Transfer (DPSL-KT)} \label{alg:DPGST2}
	\begin{algorithmic}[1]
		\INPUT  Loss function $\bar L_S$, distribution $\tilde \cD$, IGHT parameters  $s,\eta_1,\eta_2,T_1,T_2$, function $f$, $\btheta_0$,  $\sigma$
	    \STATE $\hat \btheta=\text{IGHT}(\btheta_0,\bar L_S,s,\eta_1,T_1)$
		\STATE Generate training set: $ S^{\mathrm{p}}=\{(\tilde \xb_i, y_i^{\mathrm{p}})\}_{i=1}^m$, where 
		$ y_i^{\mathrm{p}}=\la \hat \btheta,\tilde \xb_i\ra+\xi_i$, $\tilde \xb_i \sim \tilde \cD$, $\xi_i\sim N(0,\sigma^2)$
% 	$ y_i^{\mathrm{p}}=\la\tilde \xb_i,\hat \btheta\ra+\xi_i$, $\tilde \xb_i \sim \tilde \cD$, $\xi_i\sim N(0,\sigma^2)$
	\STATE Constructing the new task:  $\tilde L(\btheta)=\frac{1}{2m}\sum_{i=1}^m\big(y_i^{\mathrm{p}}-\la\btheta,\tilde \xb_i\ra\big)^2$
		\STATE $\btheta^{\mathrm{p}}=\text{IGHT}( \btheta_0,\tilde L,s,\eta_2,T_2)$
		\OUTPUT $\btheta^{\mathrm{p}}$
	\end{algorithmic}
\end{algorithm}
\begin{algorithm}[!h]
	\caption{Iterative Gradient Hard Thresholding (IGHT)} \label{alg:IGHT}
	\begin{algorithmic}[1]
		\INPUT Loss function $L_S$, parameters $s$, $\eta$, $T$, $\btheta_0$
		\FOR{$t=1,2,3,\ldots, T$}
		\STATE $\btheta_{t}=\cH_{s}\big(\btheta_{t-1}-\eta\nabla L_S(\btheta_{t-1})\big)$
		\ENDFOR
		\OUTPUT $\btheta_T$
	\end{algorithmic}
\end{algorithm}

There are two key ingredients in our framework: (1) an efficient problem solver, i.e., iterative gradient hard thresholding (IGHT) algorithm \citep{yuan2014gradient,jain2014iterative}, and (2) the knowledge transfer procedure. In detail, we first solve the optimization problem \eqref{eq:structure_opt} using IGHT, which is demonstrated in Algorithm \ref{alg:IGHT}, to get a non-private ``teacher'' estimator $\hat \btheta$. The next step is the knowledge transfer procedure: we draw some synthetic features $\{\tilde \xb_i\}_{i=1}^m$ from a given distribution $\tilde \cD$, and output the corresponding private-preserving responses $\{y_i^\mathrm{p}\}_{i=1}^m$ using the Gaussian mechanism: $y_i^\mathrm{p}=\la\hat\btheta,\tilde \xb_i\ra+\xi_i$, where $\xi_i$ is the Gaussian noise to protect the  private information contained in $\hat{\btheta}$. Finally, by solving a new sparsity constrained learning problem $\tilde L$ using the privacy-preserving synthetic dataset $S^\mathrm{p}=\{(\tilde \xb_i,y_i^\mathrm{p})\}_{i=1}^m$, we can get a differentially private ``student'' estimator $\btheta^\mathrm{p}$.

Our proposed knowledge transfer framework can achieve both strong privacy and utility guarantees. Intuitively speaking, the newly constructed learning problem can reduce the utilization of the privacy budget since we only require the generated responses to preserve the privacy of original training sample, which in turn leads to a strong privacy guarantee. In addition, this new learning problem contains the knowledge of the ``teacher'' estimator, which preserves the sparsity information of the underlying parameter. As a result, the ``student'' estimator can also have a strong utility guarantee.

\section{Main Results}\label{sec:main}
In this section, we will present the  privacy and utility guarantees for Algorithm \ref{alg:DPGST2}. We start with two conditions, which will be used in the result for generic models. Later, when we apply our result to specific models, these conditions will be verified explicitly. %and then show its applications to different examples. 

% Throughout this paper, we say a function $f$ is $\bar \beta$- smooth if for any $\btheta_1,\btheta_2\in\RR^d$, $f(\btheta_1)- f(\btheta_2)-\la\nabla f(\btheta_2),\btheta_1-\btheta_2\ra\leq \bar \beta\|\btheta_1-\btheta_2\|_2^2/2$.
% Note that we have $\cV=\{\btheta\in\RR^d:\|\btheta\|_0\leq s\}$ for the sparsity constrained problem. 
% In the following discussion, we denote the condition number of the problem by $\kappa=\beta/\mu$.
% \begin{condition}\label{con:Lip}
% Suppose for each component function $f_i$, where $i=1,\ldots,n$, we have $\|\nabla f_i(\btheta)\|_\infty \leq \gamma$ for all $\{\btheta\in\cC:\|\btheta\|_2\leq R\}$. 
% \end{condition}
The first condition is about the upper bound on the gradient of the function $L_S$, which will be used to characterize the statistical error of generic sparse models.
\begin{condition}\label{con:err}
	For a given sample size $n$ and  tolerance parameter $\zeta \in (0,1)$, let $\varepsilon(n,\zeta)$ be the smallest scalar such that with probability at least $1-\zeta$, we have $	
	\|\nabla L_S(\btheta^*)\|_\infty \leq \varepsilon(n,\zeta).
	$
% 	\begin{align*}
% 	\|\nabla L(\btheta^*)\|_\infty \leq \varepsilon(n,\rho).
% 	\end{align*}
% 	where $\varepsilon(n,\rho)$ depends on the sample size $n$ and $\rho$.
\end{condition}

To derive the utility guarantee, we also need the sparse eigenvalue condition \citep{zhang2010analysis} on the function $L_S$, which directly implies the restricted strong convex and smooth properties \citep{negahban2009unified,loh2013regularized} of the function $L_S$. %Note also that based on this condition we can get an improved utility guarantee of the Algorithm \ref{alg:DPGST2}.
% \begin{condition}\label{con:SE}
% For the function $L_S$, the Hessian matrix $\nabla^2 L_S(\btheta)$ satisfies the structured eigenvalue condition for all $\btheta$ with parameters $\mu,\beta$ such that $\mu=\inf_{\vb}\{\vb^\top\nabla^2L_S(\btheta)\vb~|~\|\vb\|_0\leq s,~\|\vb\|_2=1\}$ and $\beta=\sup_{\vb}\{\vb^\top\nabla^2L_S(\btheta)\vb~|~\|\vb\|_0\leq s,~\|\vb\|_2=1\}$.
% \end{condition}
\begin{condition}\label{con:SE}
The empirical loss $L_S$ on the training data satisfies the sparse eigenvalue condition, if for all $\btheta$, there exist positive numbers $\mu$ and $\beta$ such that
\begin{align*}
    &\mu=\inf_{\vb}\big\{\vb^\top\nabla^2L_S(\btheta)\vb~|~\|\vb\|_0\leq s,~\|\vb\|_2=1\big\},
    &\beta=\sup_{\vb}\big\{\vb^\top\nabla^2L_S(\btheta)\vb~|~\|\vb\|_0\leq s,~\|\vb\|_2=1\big\}.
\end{align*}
\end{condition}
% \begin{condition}\label{con:err_tile}
%     For the loss function $\tilde L$, we have
% 	the following holds with probability at least $1-\tilde \rho$
% 	\begin{align*}
% 	\|\nabla \tilde L(\hat \btheta)\|_{\infty} \leq \tilde \varepsilon(m,\sigma,\tilde \rho),
% 	\end{align*}
% 	where $\tilde \rho \in(0,1)$, $\hat \btheta$ is the output of Algorithm \ref{alg:DPGST2} in line 1, and $\tilde \varepsilon(m,\sigma,\tilde \rho)$ depends on the sample size $m$, noise magnitude $\sigma$, and $\tilde \rho$.
% \end{condition}

\subsection{Results for Generic Models}

We first present the privacy guarantee of Algorithm \ref{alg:DPGST2} in the setting of $(\epsilon,\delta)$-DP.
% In particular, we consider the problem with sparsity/low-rank constraint. Before we present our main results, we first define the quantity $\omega(\cC)$ as follows
% \begin{align*}
%     \omega(\cC)=\left\{
% 	\begin{array} {ll}
% 	s, & \text{if}~ \: \cC=\{\btheta\in\RR^d:\|\btheta\|_0\leq s\}, \\
% 		r, & \text{if}~\: \cC=\{\bTheta\in\RR^{d_1\times d_2}:\rank (\bTheta)\leq r\}.
% 	\end{array}
% 	\right. 
% \end{align*}
% Given this quantity, we are ready to provide our main results.

\begin{theorem}\label{thm:privacy_sparse}
%  Suppose that the loss function $\bar L$ is $\bar \beta$-smooth, and 
 Suppose the loss function on each training example satisfies $\|\nabla \ell_i(\btheta_{\min})\|_{\infty}\leq \gamma$, and $\tilde \cD$ is a sub-Gaussian distribution with parameter $\tilde \alpha$ and the covariance matrix $\|\tilde \bSigma\|_2\leq \tilde \beta$, and $m\geq C_1\tilde \alpha s\log d$ for some absolute constant $C_1$. Given a privacy budget $\epsilon$ and a constant $\delta\in(0,1)$, the output $\btheta^{\mathrm{p}}$ of Algorithm \ref{alg:DPGST2}  satisfies $(\epsilon,\delta)$-DP if $\sigma^2=8m\tilde \beta s\gamma^2\log(2.5/\delta)/(n^2\epsilon^2\lambda^2)$.
\end{theorem}

\begin{remark}
Theorem \ref{thm:privacy_sparse} suggests that in order to ensure the privacy guarantee, the only condition on the private data is the $\ell_\infty$-norm bound on the gradient of the loss function on each training example. This is in contrast to the $\ell_2$-norm bound required by many previous work \citep{kifer2012private,talwar2015nearly,wang2019differential} for sparse learning problems. We remark that $\ell_\infty$-norm bound is a milder condition than  $\ell_2$-norm bound, and gives a better utility guarantee that only depends on the $\ell_\infty$-norm of the input data vectors instead of their $\ell_2$-norm.
\end{remark}

Next, we provide the linear convergence rate and the utility guarantee of Algorithm \ref{alg:DPGST2}.
\begin{theorem}\label{thm:utility_sparse} 
% Suppose that the loss function $\bar L$ is $\bar \beta$-smooth and $L$ satisfies Condition \ref{con:err} with parameter $\varepsilon(n,\rho)$. Let $\tilde \cD$ be a sub-Gaussian distribution with parameter $\tilde \alpha$ and the covariance matrix $\|\tilde \bSigma\|_2\leq \tilde \beta$. There exist constants $\{C_i\}_{i=1}^6$ such that if $n=m\geq C_1\tilde \alpha s\log d$, $s\geq C_2\kappa^2 s^*$, where $\kappa=\bar \beta/\lambda$, then with appropriate $\eta_1,\eta_2$, $\btheta^{\mathrm{p}}$ converges to $\btheta^*$ at a linear rate. In addition, given $\epsilon,\delta$, if we choose $\lambda^2=C_3\gamma\sqrt{s^*\log d\log(1/\delta)}/(n\epsilon)$, for large enough $T_1,T_2$, the output $\btheta^{\mathrm{p}}$ of Algorithm \ref{alg:DPGST2} satisfies the following inequality with probability at least $1-\rho-C_4/d$ 
Suppose that the loss function $\bar L_S$ is $\bar \beta$-smooth and $L_S$ satisfies Condition \ref{con:err} with parameter $\varepsilon(n,\zeta)$. Under the same conditions of Theorem \ref{thm:privacy_sparse} on $\ell_i$, $\tilde \cD$, $\sigma^2$, there exist constants $\{C_i\}_{i=1}^8$ such that if $n=m\geq C_1\tilde \alpha s\log d$, $s\geq C_2\kappa^2 s^*$ with $\kappa=\bar \beta/\lambda$, the stepsize $\eta_1=C_3\lambda/\bar \beta^2,\eta_2=C_4/\tilde \beta$, then $\btheta^{\mathrm{p}}$ converges to $\btheta^*$ at a linear rate. In addition, if we choose $\lambda^2=C_5\gamma\sqrt{s^*\log d\log(1/\delta)}/(n\epsilon)$, for large enough $T_1,T_2$, with probability at least $1-\zeta-C_6/d$, the output $\btheta^{\mathrm{p}}$ of Algorithm \ref{alg:DPGST2} satisfies 
\begin{align*}
         \|\btheta^{\mathrm{p}}- \btheta^*\|_2^2 &\leq C_7\frac{s^*}{\bar \beta^2}\varepsilon(n,\zeta)^2+C_{8}\big(1/\bar\beta^2+\tilde \alpha^2/\tilde \beta\big)\frac{\gamma\sqrt{ s^{*3}\log d\log(1/\delta)}}{n\epsilon}.
\end{align*}
% where $d(\btheta^{\mathrm{p}}, \btheta^*)=\|\btheta^{\mathrm{p}}-\btheta^*\|_2^2$ for sparse problems and $d(\bTheta^{\mathrm{p}}, \bTheta^*)=\|\bTheta^{\mathrm{p}}-\bTheta^*\|_F^2$ for low-rank problems.
\end{theorem}
\begin{remark}
 The utility bound of our method consists of two terms: the first term denotes the statistical error of generic sparse models, while the second one corresponds to the error introduced by the Gaussian mechanism, and is the dominating term. Therefore, the utility bound is of order $O\big(\gamma\sqrt{s^{*3}\log d\log(1/\delta)}/(n\epsilon)\big)$, which depends on the true sparsity $s^*$ rather than the dimension of the problem $d$, and therefore is desirable for sparse learning.
\end{remark}

The following corollary shows that if $L_S$ further satisfies Condition \ref{con:SE}, our method can achieve an improved utility guarantee.
\begin{corollary}\label{thm:improve}
%  Suppose that $L$ satisfies Condition \ref{con:SE} with parameters $\mu,\beta$, each component function satisfies $\|\nabla \ell_i(\btheta_{\min})\|_{\infty}\leq \gamma$. Let $\tilde \cD$ be a sub-Gaussian distribution with parameter $\tilde \alpha$, the covariance matrix $\|\tilde \bSigma\|_2\leq \tilde \beta$, and $m\geq C_1\tilde \alpha s\log d$. Given a privacy budget $\epsilon$ and a constant $\delta\in(0,1)$, the output $\btheta^{\mathrm{p}}$ of Algorithm \ref{alg:DPGST2} satisfies $(\epsilon,\delta)$-DP if we set $\lambda=0$ and $\sigma^2=8m\tilde \beta s\gamma^2\log(2.5/\delta)/(n^2\epsilon^2\mu^2)$. In addition, if $L$ satisfies Condition \ref{con:err} with parameter $\varepsilon(n,\rho)$, $n=m$, $s>C_2\kappa^2s^*$, where $\kappa=\beta/\mu$, then with appropriate $\eta_1,\eta_2$, and large enough $T_1,T_2$, the following inequality 
Suppose that $L_S$ satisfies Condition \ref{con:SE} with parameters $\mu,\beta$. Under the same conditions of Theorem \ref{thm:utility_sparse} on $L_S,\ell_i,\tilde \cD$, the output $\btheta^{\mathrm{p}}$ of Algorithm \ref{alg:DPGST2} satisfies $(\epsilon,\delta)$-DP if we set $\lambda=0$ and $\sigma^2=8m\tilde \beta s\gamma^2\log(2.5/\delta)/(n^2\epsilon^2\mu^2)$. In addition, there exist constants $\{C_i\}_{i=1}^7$ such that if $n=m\geq C_1\tilde \alpha s\log d$, $s\geq C_2\kappa^2s^*$ with $\kappa=\beta/\mu$, step size $\eta_1=C_3\mu/\beta^2,\eta_2=C_4/\tilde \beta$, for large enough $T_1,T_2$, with probability at least $1-\zeta-C_5/d$, the output $\btheta^{\mathrm{p}}$ of Algorithm \ref{alg:DPGST2} satisfies
 \begin{align*}
         \|\btheta^{\mathrm{p}}- \btheta^*\|_2^2 &\leq C_6\frac{s^*}{ \beta^2}\varepsilon(n,\zeta)^2+C_{7}\tilde \alpha^2\frac{\gamma^2 s^{*2}\log d\log(1/\delta)}{\tilde \beta\mu^2n^2\epsilon^2}.
\end{align*}
\end{corollary}
\begin{remark}
Corollary \ref{thm:improve} shows that if the training loss on the private data satisfies the sparse eigenvalue condition, Algorithm \ref{alg:DPGST2} can 
% still provide the desired privacy guarantee. In addition, our method can 
achieve $\tilde O\big(\gamma ^2s^{*2}/(n\epsilon)^2\big)$ utility guarantee by setting $\lambda=0$ and the variance $\sigma^2$ accordingly. It improves the utility without the sparse eigenvalue condition $\tilde O\big(\gamma s^{*3/2}/(n\epsilon)\big)$ in Theorem \ref{thm:utility_sparse} by a factor of $\tilde O\big(n\epsilon/\gamma \sqrt{s^*}\big)$. Note that sparse eigenvalue condition has been verified for many sparse models \citep{negahban2009unified} including sparse linear regression and sparse logistic regression. %under this assumption on the private data. 
\end{remark}

\subsection{Results for Specific Models}
In this subsection, we demonstrate the results of our framework for specific models. Note that the privacy guarantee has been established in Theorem \ref{thm:privacy_sparse}, and we only present the utility guarantees.

\subsubsection{Sparse linear regression}
%\noindent\textbf{Sparse Linear Regression:} 
We consider the following linear regression problem in the high-dimensional regime \citep{tibshirani1996regression}: $  \yb = \Xb \btheta^* + \bxi$,
% \begin{align*}
%   \yb = \Xb \btheta^* + \bvarepsilon, 
% \end{align*}
% $\yb = \Xb \btheta^* + \bvarepsilon$,
where $\yb \in \RR^n$ is the response vector, $\Xb \in \RR^{n\times d}$ denotes the design matrix, $\bxi \in \RR^n$ is a noise vector, and $\btheta^* \in \RR^d$ with $\|\btheta^*\|_0\leq s^*$ is the underlying sparse coefficient vector that we want to recover. %A widely used estimator for the sparse linear regression problem is Lasso \citep{tibshirani1996regression}. However, the $\ell_1$-regularization in Lasso often incurs large estimation bias compared with the $\ell_0$-constrained estimators \citep{fan2001variable,zhang2010nearly}. 
In order to estimate the sparse vector $\btheta^*$, we consider the following sparsity constrained estimation problem, which has been studied in many previous work \citep{zhang2011adaptive,foucart2013mathematical,yuan2014gradient,jain2014iterative,chen2016accelerated} 
\begin{align}\label{eq:linear regression}
\min_{\btheta\in\RR^d} \frac{1}{2n}\|\Xb\btheta-\yb\|^{2}_{2}+\frac{\lambda}{2}\|\btheta\|_2^2~~\text{subject to}~~\|\btheta\|_{0}\leq s.
\end{align}
% where $\lambda\geq 0$ is a regularization parameter, and $s$ is a tuning parameter for the sparsity of $\btheta$.
The utility guarantee of Algorithm \ref{alg:DPGST2} for solving \eqref{eq:linear regression} can be implied by Theorem \ref{thm:utility_sparse}. Here we only need to verify Condition \ref{con:SE} for the sparse linear regression model. In specific, we can show that $\nabla L_S(\btheta^*)=\Xb^\top\bxi/n$, and we can prove that (See Lemma \ref{lemma:sr4lasso} in Appendix) $\|\nabla L_S(\btheta^*)\|_\infty \leq C_1\nu\sqrt{\log d/n}$ holds with probability at least $1-\exp(-C_2n)$, where $C_1, C_2$ are absolute constants. Therefore, we have $\zeta=1-\exp(-C_2n)$, $\varepsilon(n,\zeta)=C_1\nu\sqrt{\log d/n}$. By substituting these quantities into Theorem \ref{thm:utility_sparse}, we can obtain the following corollary.

\begin{corollary}\label{corollary:lasso}
% Suppose that each row of the design matrix satisfies $\max_{i \in [n]}\|\xb_i\|_{\infty}\leq K$, and the noise vector $\bvarepsilon\sim N(0,\nu^2\Ib_n)$. Let $\tilde \cD$ be a sub-Gaussian distribution with parameter $\tilde \alpha$ and the covariance matrix $\|\tilde \bSigma\|_2\leq \tilde \beta$. For a given $\epsilon,\delta$, there exist constants $\{C_i\}_{i=1}^5$ such that if $m=n\geq C_1 s\log d$, $\sigma^2=8s\gamma^2\log(2.5/\delta)/(n\epsilon^2\lambda^2)$, $\lambda^2=C_2K^2s^{*}\sqrt{\log d\log(1/\delta)}/(n\epsilon)$, running Algorithm \ref{alg:DPGST2} with appropriate $\eta_1,\eta_2$ and large enough $s,T_1,T_2$, then with probability at least $1-C_3/d$,
Suppose that each row of the design matrix satisfies $\max_{i \in [n]}\|\xb_i\|_{\infty}\leq K$, and the noise vector $\bxi\sim N(0,\nu^2\Ib_n)$. Under the same conditions of Theorem \ref{thm:utility_sparse} on $\tilde \cD,\sigma^2,\eta_1,\eta_2,s$, there exist constants $\{C_i\}_{i=1}^5$ such that if $m=n\geq C_1 s\log d$, $\lambda^2=C_2K^2s^{*}\sqrt{\log d\log(1/\delta)}/(n\epsilon)$, with probability at least $1-C_3/d$, the output $\btheta^{\mathrm{p}}$ of Algorithm \ref{alg:DPGST2} 
%with appropriate $\eta_1,\eta_2$ and
%large enough $T_1,T_2$, then 
satisfies
\begin{align*}
   \|\btheta^{\mathrm{p}}-\btheta^*\|_2^2
    \leq C_{4}\nu^2K^2\frac{s^*\log d}{n}+C_{5}\tilde\alpha^2K^2 \frac{ s^{*2}\sqrt{\log d \log(1/\delta)}}{\tilde \beta n\epsilon}.
\end{align*}
\end{corollary}

\begin{remark}
Corollary \ref{corollary:lasso} suggests that 
% after $T=O\big(\log(n^2\epsilon^2/s)\big)$ number of iterations, 
$ O\big(s^*\log d/n+K^2s^{*2}\sqrt{\log d\log(1/\delta)}/(n\epsilon)\big)$ utility guarantee can be achieved by our algorithm. The term $O(s^*\log d/n)$ denotes the statistical error for sparse vector estimation, which matches the minimax lower bound \citep{raskutti2011minimax}. While the term $ \tilde O(K^2s^{*2}/(n\epsilon))$ corresponds to the error introduced by the privacy-preserving mechanism, and is the dominating term. Compared with the best-known result \citep{kifer2012private,wang2019differential}  $\tilde O(\tilde K^2s^{*2}/(n^2\epsilon^2))$, where $\|\xb_i\|_2\leq \tilde K$ for all $i\in[n]$, our utility guarantee does not require the sparse eigenvalue condition and is better than their results by a factor of $\tilde O\big(\tilde K^2/(K^2n\epsilon)\big)$. Since we have $\tilde K\leq \sqrt{d}K$ in the worst case, the improvement factor can be as large as $\tilde O\big(d/(n\epsilon)\big)$.
Compared with the utility guarantee $\tilde O\big(1/(n\epsilon)^{2/3}\big)$ obtained by \citet{talwar2015nearly}, our method improves their result by a factor of $ \tilde O\big((n\epsilon)^{1/3}/(Ks^{*})^2\big)$, which demonstrates the advantage of our framework.
\end{remark}
Next, we present the theoretical guarantees of our methods under the extra sparse eigenvalue condition for sparse linear regression.
\begin{corollary}\label{corollary:lasso1}
Suppose that each row $\xb_i$ of the design matrix satisfies $\xb_i\sim N(0,\bSigma)$, $\max_{i \in [n]}\|\xb_i\|_{\infty}\leq K$, and the noise vector $\bxi\sim N(0,\nu^2\Ib_n)$. For a given $\epsilon,\delta$, under the same conditions of Corollary \ref{thm:improve} on $\tilde \cD,\sigma^2,\lambda,\eta_1,\eta_2,s$, there exist constants $\{C_i\}_{i=1}^4$ such that if $m=n\geq C_1 s\log d$, the output of Algorithm \ref{alg:DPGST2} satisfies $(\epsilon,\delta)$-DP. In addition, with probability at least $1-C_2/d$, we have
\begin{align*}
   \|\btheta^{\mathrm{p}}-\btheta^*\|_2^2
    \leq C_{3}\nu^2K^2\frac{s^*\log d}{n}+C_{4}\tilde\alpha^2K^2 \frac{ s^{*3}\log d \log(1/\delta)}{\tilde \beta n^2\epsilon^2}.
\end{align*}
\end{corollary}
\begin{remark}\label{remark1}
According to Corollary \ref{corollary:lasso1}, the output of Algorithm \ref{alg:DPGST2} will satisfy $(\epsilon,\delta)$-DP with the utility guarantee $\tilde O\big(K^2s^{*3}/(n^2\epsilon^2)\big)$, which improves the result in Corollary \ref{corollary:lasso} by a factor of $\tilde O\big(n\epsilon/s^{*}\big)$. 
\end{remark}

\subsubsection{Sparse logistic regression}
%\noindent\textbf{Sparse Logistic Regression:} 
	For high-dimensional logistic regression, we assume the label of each example follows an i.i.d. Bernoulli distribution conditioned on the input vector $	\PP(y=1|\xb,\btheta^*)=\exp\big(\btheta^{*\top}\xb-\log\big(1+\exp (\btheta^{*\top}\xb)\big)\big)$,
% 	\begin{align*}%\label{eq:GLM}
% 	\PP(y=1|\xb,\btheta^*)=\exp\big(\btheta^{*\top}\xb-\log\big(1+\exp (\btheta^{*\top}\xb)\big)\big),
% 	\end{align*}
% $	\PP(y=1|\xb,\btheta^*)=\exp\big(\la\btheta^*,\xb\ra-\log\big(1+\exp (\btheta^{*\top}\xb)\big)\big)$,
where $\xb \in \RR^d$ is the input vector, $\btheta^*\in\RR^d$ with $\|\btheta^*\|_0\leq s^*$ is the sparse parameter vector we would like to estimate. Given observations $\{(\xb_i,y_i)\}_{i=1}^n$, 
% a widely used estimator for $\btheta^*$ is the $\ell_1$ regularized maximum likelihood estimator \cite{negahban2009unified,loh2013regularized}. 
we consider the following maximum likelihood estimation problem with sparsity constraints \citep{yuan2014gradient,chen2016accelerated}
	\begin{align}\label{eq:constrained GLM}
	\min_{\btheta\in\RR^d}
-\frac{1}{n}&\sum_{i=1}^n \big[ y_i\btheta^\top\xb_i -\log\big(1+\exp (\btheta^{\top}\xb_i)\big)\big]+\frac{\lambda}{2}\|\btheta\|_2^2
~~\text{subject to}~~\|\btheta\|_{0}\leq s.
	\end{align}
% $\lambda$ is a regularization parameter, $s$ controls the sparsity of $\btheta$.
The utility guarantee of Algorithm \ref{alg:DPGST2} for solving \eqref{eq:constrained GLM} is shown in the following corollary.

\begin{corollary}\label{corollary:logistic}
% Suppose that each $\xb_i$ satisfies $\max_{i\in[n]}\|\xb_i\|_{\infty}\leq K$. Let $\tilde \cD$ be a sub-Gaussian distribution with parameter $\tilde \alpha$ and the covariance matrix $\|\tilde \bSigma\|_2\leq \tilde \beta$. For a given $\epsilon,\delta$, there exist constants $\{C_i\}_{i=1}^4$ such that if $m=n\geq C_1s\log d$, $\sigma^2=8s\gamma^2\log(2.5/\delta)/(n\epsilon^2\lambda^2)$, $\lambda^2=C_2K\sqrt{s^{*}\log d\log(1/\delta)}/(n\epsilon)$, running Algorithm \ref{alg:DPGST2} with appropriate $\eta_1,\eta_2$ and large enough $s,T_1,T_2$, then with probability at least $1-C_3/d$, we have the following utility for $\btheta^{\mathrm{p}}$
Under the same conditions of Corollary \ref{corollary:lasso} on $\xb_i, \tilde \cD,\sigma^2,\eta_1,\eta_2,s$, there exist constants $\{C_i\}_{i=1}^5$ such that if $m=n\geq C_1s\log d$, $\lambda^2=C_2K\sqrt{s^{*}\log d\log(1/\delta)}/(n\epsilon)$, with probability at least $1-C_3/d$, the output $\btheta^{\mathrm{p}}$ of Algorithm \ref{alg:DPGST2} %with appropriate $\eta_1,\eta_2$ and large enough $s,T_1,T_2$, 
satisfies
\begin{align*}
 \|\btheta^{\mathrm{p}}- \btheta^*\|_2^2
    \leq C_4K^2\frac{s^*\log d}{n}+C_{5}\tilde\alpha^2K \frac{ \sqrt{s^{*3}\log d \log(1/\delta)}}{\tilde \beta n\epsilon}. 
\end{align*}
\end{corollary}

\begin{remark}
Corollary \ref{corollary:logistic} suggests that $ O\big(s^*\log d/n+K\sqrt{s^{*3}\log d\log(1/\delta)}/(n\epsilon)\big)$ utility guarantee can be obtained by our algorithm for sparse logistic regression. The term $\tilde O\big(Ks^{*3/2}/(n\epsilon))$ caused by the Gaussian mechanism is the dominating term and does not depend on the sparse eigenvalue condition, and is better than the best-known result \citep{wang2019differential} $\tilde O\big(\tilde K^2s^{*2}/(n^2\epsilon^2)\big)$ by a factor of $\tilde O\big(\tilde K^2s^{*1/2}/(Kn\epsilon)\big)$. The improvement factor can be as large as $\tilde O\big(dK/(n\epsilon)\big)$ since $\tilde K\leq \sqrt{d}K$. %As discussed in Remark \ref{remark1}, if we have the extra data assumption such that $L$ satisfies Condition \ref{con:SE}, we can get the utility guarantee $\tilde O\big(K^2s^{*2}/(n^2\epsilon^2)\big)$, which improve the result in Corollary \ref{corollary:logistic} by a factor of $\tilde O\big(n\epsilon/(K\sqrt{s^*})\big)$.
\end{remark}

If we have the extra sparse eigenvalue condition, our method can achieve an improved utility guarantee for sparse logistic regression as follows.
\begin{corollary}\label{corollary:logistic1}
Suppose that each row $\xb_i$ of the design matrix satisfies $\xb_i\sim N(0,\bSigma)$, $\max_{i \in [n]}\|\xb_i\|_{\infty}\leq K$. For a given $\epsilon,\delta$, under the same conditions of Corollary \ref{thm:improve} on $\tilde \cD,\sigma^2,\lambda,\eta_1,\eta_2,s$, there exist constants $\{C_i\}_{i=1}^4$ such that if $m=n\geq C_1s\log d$, the output of Algorithm \ref{alg:DPGST2} satisfies $(\epsilon,\delta)$-DP. In addition, with probability at least $1-C_2/d$, we have the following utility for $\btheta^{\mathrm{p}}$
\begin{align*}
 \|\btheta^{\mathrm{p}}- \btheta^*\|_2^2
    \leq C_3K^2\frac{s^*\log d}{n}+C_{4}\tilde\alpha^2K^2s^{*2} \frac{ \log d \log(1/\delta)}{\tilde \beta n^2\epsilon^2}. 
\end{align*}
\end{corollary}
\begin{remark}
Corollary \ref{corollary:logistic1} shows that our method can obtain an improved utility guarantee $\tilde O\big(K^2s^{*2}/(n\epsilon)^2\big)$ for sparse logistic regression under the extra sparse eigenvalue assumption. 
\end{remark}
\section{Numerical Experiments}\label{sec:exp}
In this section, we present experimental results of our proposed algorithm on both synthetic and real datasets. For sparse linear regression, we compare our framework with Two stage \citep{kifer2012private}, Frank-Wolfe \citep{talwar2015nearly}, and DP-IGHT \citep{wang2019differential} algorithms. For sparse logistic regression, we compare our framework with DP-IGHT \citep{wang2019differential} algorithm. 
% Although there exist several works \citep{chaudhuri2009privacy,jain2014near} about the differentially private logistic regression, they did not consider the sparsity constraint in their problems. Therefore, we do not compare with their methods in our experiments. 
For all of our experiments, we choose the parameters of different methods according to the requirements of their theoretical guarantees. More specifically, on the synthetic data experiments, we assume $s^*$ is known for all the methods. On the real data experiments, $s^*$ is unknown, neither our method or the competing methods has the knowledge of $s^*$. So we simply choose a sufficiently large $s$ as a surrogate of $s^*$. Given $s$, for the parameter $\lambda$ in our method, according to Theorem 4.5, we choose $\lambda$ from a sequence of values $c_1\sqrt{s\log d\log(1/\delta)}/(n\epsilon)$, where $c_1\in\{10^{-6},10^{-5},\ldots,10^{1}\}$, by cross-validation. For competing methods, given $s$, we choose the iteration number of Frank-Wolfe from a sequence of values $c_2s$, where $c_2\in\{0.5,0.6,\ldots,1.5\}$, and the regularization parameter in the objective function of Two Stage from a sequence of values $c_3s/\epsilon$, where $c_3\in\{10^{-3},10^{-2},\ldots,10^{2}\}$, by cross-validation. For DP-IGHT, we choose its stepsize from the grid $\{1/2^0,1/2^1,\ldots,1/2^6\}$ by cross-validation. For the non-private baseline, we use the non-private IGHT \citep{yuan2014gradient}.

\subsection{Numerical Simulations}
In this subsection, we investigate our framework on synthetic datasets for sparse linear and logistic regression. In both problems, we generate the design matrix $\Xb\in\RR^{n\times d}$ such that each entry is drawn i.i.d. from a uniform distribution $U(-1,1)$, and the underlying sparse vector $\btheta^*$ has $s$ nonzero entries that are randomly generated. In addition, we consider the following two settings: (i) $n=800,d=1000,s^*=10$; (ii) $n=4000,d=5000,s^*=50$. We choose $\tilde \cD$ to be a uniform distribution $U(-1,1)$, which implies $\tilde \beta=1/3$.
\begin{figure*}[!ht]%
  \centering
  %\vspace{.3in}
   \subfigure[Linear regression]{
    \label{fig2:subfig:1.a} %% label for first subfigure
    \includegraphics[width=0.23\textwidth]{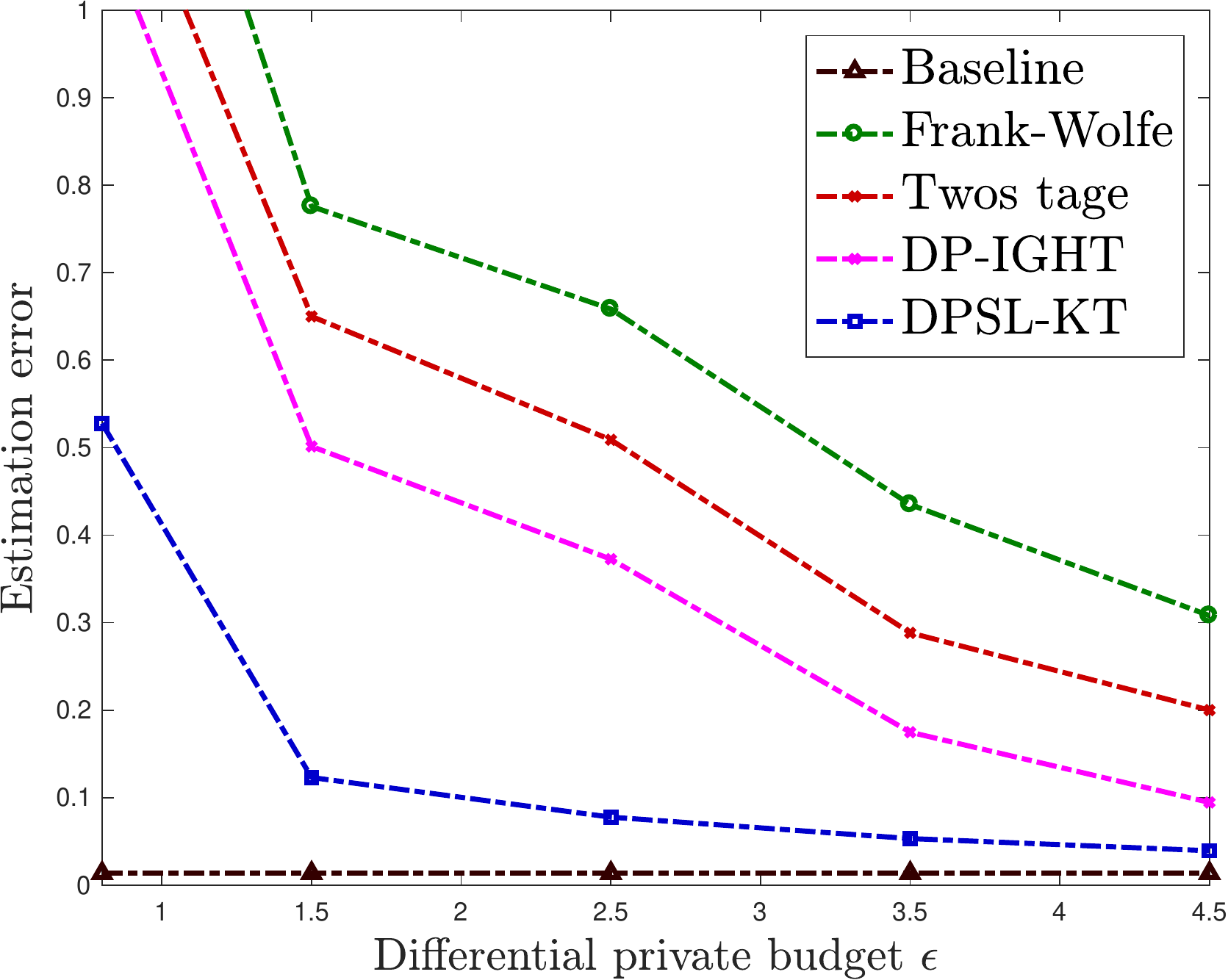}}
    \subfigure[Linear regression]{
    \label{fig2:subfig:1.b} %% label for first subfigure
    \includegraphics[width=0.23\textwidth]{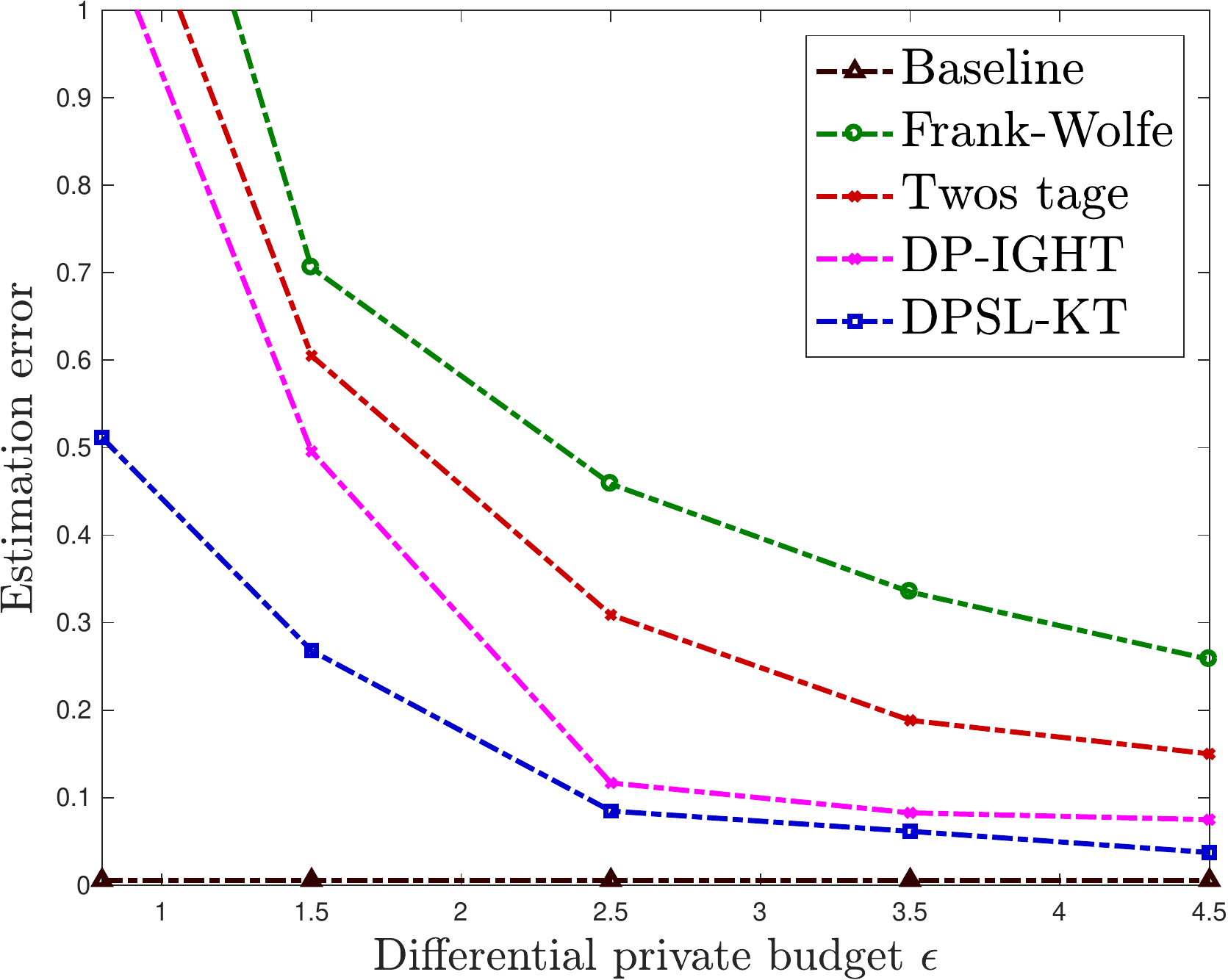}}
    \subfigure[Logistic regression]{
    \label{fig2:subfig:1.c} %% label for first subfigure
    \includegraphics[width=0.23\textwidth]{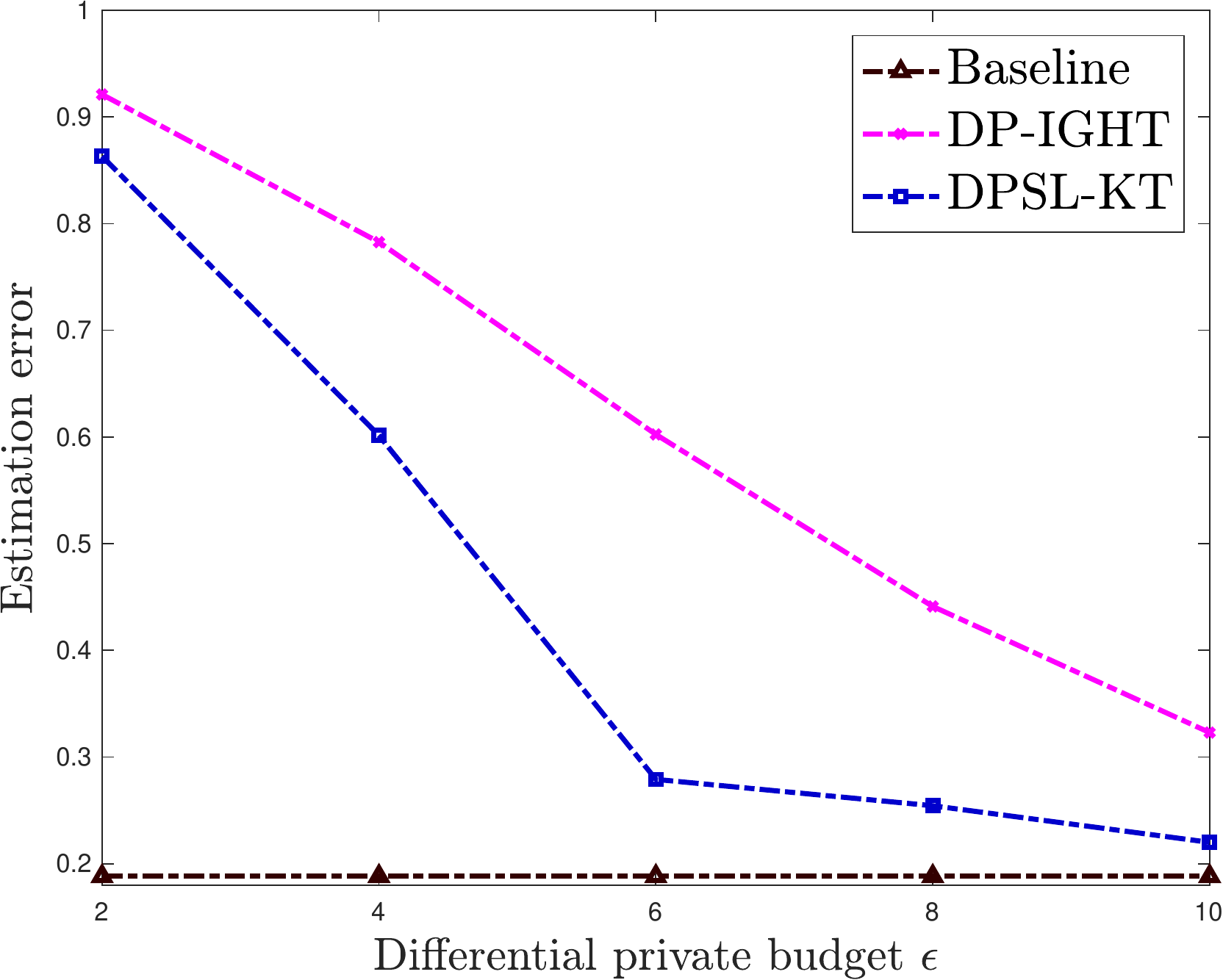}}
    \subfigure[Logistic regression]{
    \label{fig2:subfig:1.d} %% label for first subfigure
    \includegraphics[width=0.23\textwidth]{log2.pdf}}
   
  %\vspace{.3in}
  \caption{Numerical results for sparse linear and logistic regression. (a), (b) Reconstruction error versus privacy budget for sparse linear regression; (c), (d) Reconstruction error versus privacy budget for sparse linear regression. } \label{fig:mc}%% label for entire figure
\end{figure*}
%\subsubsection{Sparse Linear Regression}
\begin{table*}[!ht]
\small
	\caption{Comparison of different algorithms for various privacy budgets $\epsilon$ with $\delta=10^{-5}$ in terms of MSE (mean $\pm$ std) and its corresponding standard deviation on E2006-TFIDF.}
	\label{table:e2006}
	\centering
	\begin{tabular}{l|c|c|c|c|c}
		\toprule
		Method & $\epsilon=0.8$ & $\epsilon=1.5$ & $\epsilon=2.5$ & $\epsilon=3.5$ & $\epsilon=4.5$\\
		%\cline{2-11}
		\midrule
		IGHT         &  0.8541  &    0.8541 &    0.8541  &  0.8541    &    0.8541  \\
		Frank-Wolfe      & 4.471 (0.239)   & 2.004 (0.155)  &1.535 (0.140) &   1.206 (0.095) &     1.099 (0.082) \\
		Two stage     &4.022 (0.159)  &  1.803 (0.141) &  1.326 (0.093)  & 1.107 (0.103)  &   1.053 (0.069)  \\
		DP-IGHT       & 3.731 (0.207)  &  1.687 (0.126)    &  1.304 (0.035)   & 1.067 (0.051)  & 0.968 (0.062)\\
		\textbf{DPSL-KT}      & 1.227 (0.110)  & 1.178 (0.056)     & 1.065 (0.054)  & 0.971 (0.031)  &0.952 (0.010) \\
		\hline
	\end{tabular}
\end{table*}
\begin{table*}[!ht]
\small
		\caption{Comparison of different algorithms for various privacy budgets $\epsilon$ with $\delta=10^{-5}$ in terms of test error (mean $\pm$ std) and its corresponding standard deviation on RCV1 data.}\label{table:real2}
	\centering
	\begin{tabular}{l|c|c|c|c}
		\toprule
		Method & $\epsilon=2$ & $\epsilon=4$ & $\epsilon=6$ & $\epsilon=8$ \\
		%\cline{2-11}
		\midrule
		IGHT         &  0.0645   &    0.0645 &    0.0645  &  0.0645     \\
		Frank-Wolfe      & 0.1381 (0.0045)   & 0.1134 (0.0041)  &0.0978 (0.0032) &   0.0882 (0.0033) \\
		Two stage     &0.1272 (0.0044)  &  0.1061(0.0038) &  0.0949 (0.0035)  & 0.0866 (0.0031)   \\
		DP-IGHT       & 0.1179 (0.0035)  &   0.1026 (0.0036)   &  0.0922 (0.0032) & 0.0824 (0.0029)  \\
				\textbf{DPSL-KT}       & 0.1105 (0.0038)  &   0.0974 (0.0035)   &  0.0885 (0.0029) & 0.0787(0.0031) \\
		\hline
	\end{tabular}
 \end{table*}
 
\noindent\textbf{Sparse linear regression}
For sparse linear regression,  the observations are generated according to the linear regression model $\yb=\Xb^\top\btheta^*+\bxi$, where the noise vector $\bxi\sim N(0,\nu^2\Ib)$ with $\nu^2=0.1$.  In our experiments, we set $\delta=0.01$ and vary the privacy budget $\epsilon$ from $0.8$ to $5$. Note that due to the hardness of the problem itself, we choose relatively large privacy budgets compared with the low-dimensional problem to ensure meaningful results. Figure \ref{fig2:subfig:1.a} and \ref{fig2:subfig:1.b}  illustrate the estimation error $\|\hat \btheta-\btheta^*\|_2/\|\btheta^*\|_2$ of different methods averaged over 10 trails. The results show that the estimation error of our method is close to the non-private baseline, and is significantly better than other private baselines. Even when we have a small privacy budget (i.e., $\epsilon=0.8$), our method can still recover the underlying sparse vector with reasonably small estimation error, while others fail.

%\subsubsection{Sparse Logistic Regression}
\noindent\textbf{Sparse logistic regression}
For sparse logistic regression, each label is generated from the logistic distribution $\PP(y=1)=1/\big(1+\exp(\xb_i^\top\btheta^*)\big)$. In this problem, we vary the privacy budget $\epsilon$ from $2$ to $10$, and set $\delta=0.01$.  We present the estimation error versus privacy budget $\epsilon$ of different methods in Figure \ref{fig2:subfig:1.c} and \ref{fig2:subfig:1.d}. The results show that our method can output accurate estimators when we have relative large privacy budget, and it consistently outperforms the private baseline.

\subsection{Real Data Experiments} 
For real data experiments, we use E2006-TFIDF dataset \citep{kogan2009predicting} and RCV1 dataset \citep{lewis2004rcv1}, for the evaluation of sparse linear regression and sparse logistic regression, respectively. 
% Due to space limit, we defer the experimental results for sparse logistic regression to the supplemental material.

\noindent\textbf{E2006-TFIDF data}
For sparse linear regression problem, we use E2006-TFIDF dataset, which consists of financial risk data from thousands of U.S. companies. In detail, it contains 16087 training
examples, 3308 testing examples, and we randomly sample
25000 features for this experiment. In order to validate our proposed framework, we randomly divide the original dataset into two  datasets: private dataset and public dataset. For the private dataset, it contains $8044$ training examples, and we assume that this dataset contains the sensitive information that we want to protect. For the public dataset, it contains $8043$ training examples. We set $s = 2000$, $\delta=10^{-5}$, $\epsilon \in [0.8,5]$. We estimate $\tilde \beta$ by the sample covariance matrix.  Table \ref{table:e2006} reports the mean
square error (MSE) on the test data of different methods for various privacy budgets over 10 trails. 
% In specific, MSE on the test data is defined as follows: $\|\Xb_{\text{test}}^\top\hat\btheta-\yb_{\text{test}}\|_2^2/(2n_{\text{test}})$, where $\{\Xb_{\text{test}},\yb_{\text{test}}\}$ are the test data, $n_{\text{test}}$ is the number of test examples, and $\hat \btheta$ is the estimator learned on the training data.
The results show that the performance of our algorithm is close to the non-private baseline even when we have small private budgets, and is much better than existing methods.

% \vspace{-0.1in}
\noindent\textbf{RCV1 data}
For sparse logistic regression, we use a Reuters Corpus Volume I (RCV1) data set for text categorization research. RCV1 is released by Reuters, Ltd. for research purposes, and consists of
over 800000 manually categorized newswire stories. It contains 20242 training examples, 677399 testing examples and 47236 features. As before, we randomly divide the original dataset into two datasets with equal size serving as the private and publice datasets. In addition, we randomly choose 10000 test examples and 20000 features, and set $s = 500$, $\delta=10^{-5}$, $\epsilon \in [2,8]$. We estimate $\tilde \beta$ by the sample covariance matrix. We compare all algorithms in terms of their classification error on the test set over 10 replications, which is summarized in Table \ref{table:real2}. Evidently our
algorithm achieves the lowest test error among all private algorithms on RCV1 dataset, which demonstrates
the superiority of our algorithm.

\section{Conclusions and Future Work}
In this paper, we developed a differentially private framework for sparse learning using the idea of knowledge transfer. We establish the linear convergence rate and the utility guarantee of our method. Experiments on both synthetic and real-world data demonstrate the superiority of our algorithm. For the future work, it is very interesting to generalize our framework to other structural constrained learning problems such as the low-rank estimation problem. It is also very interesting to study the theoretical lower-bound of the differentially private sparse learning problem to access the optimality of our proposed method.

\appendix
\section{Additional Results}
In this section, we present the additional theoretical guarantees of our methods under the extra sparse eigenvalue conditions for sparse linear and logistic regression.
\subsection{Additional Main Results}
\begin{corollary}\label{corollary:lasso1}
Suppose that each row $\xb_i$ of the design matrix satisfies $\xb_i\sim N(0,\bSigma)$, $\max_{i \in [n]}\|\xb_i\|_{\infty}\leq K$, and the noise vector $\bxi\sim N(0,\nu^2\Ib_n)$. For a given $\epsilon,\delta$, under the same conditions of Corollary \ref{thm:improve} on $\tilde \cD,\sigma^2,\lambda,\eta_1,\eta_2,s$, there exist constants $\{C_i\}_{i=1}^4$ such that if $m=n\geq C_1 s\log d$, the output of Algorithm \ref{alg:DPGST2} satisfies $(\epsilon,\delta)$-DP. In addition, with probability at least $1-C_2/d$, we have
\begin{align*}
   \|\btheta^{\mathrm{p}}-\btheta^*\|_2^2
    \leq C_{3}\nu^2K^2\frac{s^*\log d}{n}+C_{4}\tilde\alpha^2K^2 \frac{ s^{*3}\log d \log(1/\delta)}{\tilde \beta n^2\epsilon^2}.
\end{align*}
\end{corollary}
\begin{remark}
According to Corollary \ref{corollary:lasso1}, we can achieve an improved utility guarantee $\tilde O\big(K^2s^{*3}/(n\epsilon)^2\big)$ for sparse linear regression if we have further assumption, i.e., Gaussian distribution, on the private data $\xb_i$.
\end{remark}

\begin{corollary}\label{corollary:logistic1}
Suppose that each row $\xb_i$ of the design matrix satisfies $\xb_i\sim N(0,\bSigma)$, $\max_{i \in [n]}\|\xb_i\|_{\infty}\leq K$. For a given $\epsilon,\delta$, under the same conditions of Corollary \ref{thm:improve} on $\tilde \cD,\sigma^2,\lambda,\eta_1,\eta_2,s$, there exist constants $\{C_i\}_{i=1}^4$ such that if $m=n\geq C_1s\log d$, the output of Algorithm \ref{alg:DPGST2} satisfies $(\epsilon,\delta)$-DP. In addition, with probability at least $1-C_2/d$, we have the following utility for $\btheta^{\mathrm{p}}$
\begin{align*}
 \|\btheta^{\mathrm{p}}- \btheta^*\|_2^2
    \leq C_3K^2\frac{s^*\log d}{n}+C_{4}\tilde\alpha^2K^2s^{*2} \frac{ \log d \log(1/\delta)}{\tilde \beta n^2\epsilon^2}. 
\end{align*}
\end{corollary}
\begin{remark}
Corollary \ref{corollary:logistic1} shows that if we have further assumption, i.e., Gaussian distribution, on the private data $\xb_i$, we can obtain an improved utility guarantee $\tilde O\big(K^2s^{*2}/(n\epsilon)^2\big)$ for sparse linear logistic regression. 
\end{remark}

\section{Proofs of the Main Results}
\subsection{Proof of Theorem \ref{thm:privacy_sparse}}
In this subsection, we will derive the differential privacy of Algorithm \ref{alg:DPGST2}. 
% First, we need the following lemma, which provides the convergence guarantee of IGHT, i.e., Algorithm \ref{alg:IGHT}.
% \begin{lemma}\label{lemma:IGHT_converge}
% Consider the regularized sparsity constrained problem \eqref{eq:structure_opt}. Suppose the loss function is $\bar \beta$-smooth, if we choose appropriate $\eta$ and $s$, the output $\hat \btheta$ of Algorithm \ref{alg:IGHT} satisfies the following
% \begin{align*}
%     \|\btheta_T-\btheta^{\min}\|_2^2\leq \rho^T \|\btheta_0-\btheta^{\min}\|_2^2,
% \end{align*}
% where $\rho<1$ and $\btheta^{\min}$ is the minimizer of \eqref{eq:structure_opt}.
% \end{lemma}
First, we need the following lemma to characterize the properties of the generated samples. It has been previously proved for many common examples of sub-Gaussian random design \cite{raskutti2011minimax,agarwal2010fast,rudelson2012reconstruction}.
\begin{lemma}\label{lemma:SE4lasso1}
Suppose each row of the design matrix $\tilde \Xb\in\RR^{m\times d}$ follows sub-Gaussian distribution with parameter $\tilde \alpha$, and the covariance matrix $\|\tilde \bSigma\|_2\leq \tilde \beta$, there exist some constants $\{C_i\}_{i=1}^2$ such that for all $\vb\in\RR^{d}$ with at most $s$ nonzero entries, if $m\geq C_1s\alpha^2\log d$, with probability at least $1-\exp(-C_2m)$, we have
\begin{align*}
       \psi_1\tilde \beta \|\vb\|_2^2 \leq\frac{\|\tilde \Xb\vb\|_2^2}{m}\leq \psi_2\tilde \beta\|\vb\|_2^2,
\end{align*}
where $\psi_1=4/5$ and $ \psi_2=6/5$.
\end{lemma}

\begin{proof}[Proof of Theorem \ref{thm:privacy_sparse}]
Note that there is no privacy issue with respect to the newly generated features $\tilde \xb_i\in\RR^d$ for $i=1,\ldots,m$. We only need to prove that the generated predictions $y_1^{\mathrm{p}},\ldots, y_m^{\mathrm{p}}$ satisfy differential privacy. Thus by the post-processing property, i.e., Lemma  \ref{lemma:com_post}, we can show that the output $ \btheta^{\mathrm{p}}$ of Algorithm \ref{alg:DPGST2} satisfies differential privacy.

% \textbf{Sparsity Constraint:} 
According to Algorithm \ref{alg:DPGST2}, we generate the new training set $S^{\mathrm{p}}$ with $i$-th example as $(y_i^{\mathrm{p}},\tilde \xb_i)$, where $y_i^{\mathrm{p}}=\la\hat \btheta,\tilde \xb_i\ra+\xi_i$, $\tilde \xb_i\sim \tilde \cD, \xi_i\sim N(0,\sigma^2)$. Consider the following function $\qb:\cS^n\rightarrow\RR^m$ such that the $i$-th coordinate of $\qb(S)$ is $\la\hat \btheta_S,\tilde \xb_i\ra$, where $\hat \btheta_S$ is trained on the training set $S$ using IGHT, i.e., Algorithm \ref{alg:IGHT}. Thus for the function $\qb$, we can characterize its sensitivity as follows: for two adjacent training sets $S,S^\prime$ with one different example indexed by $i$, we have
\begin{align}\label{eq:sens_1}
      \Delta(\qb)&=\sqrt{\sum_{i=1}^m\big(\la\hat \btheta_S,\tilde \xb_i\ra-\la\hat \btheta_{S^\prime},\tilde \xb_i\ra\big)^2}\nonumber\\
      &=\sqrt{\sum_{i=1}^m\la\hat \btheta_S-\hat \btheta_{S^{\prime}},\tilde \xb_i\ra^2}\nonumber\\
      &\leq \sqrt{2m\tilde \beta}\big\|\hat \btheta_S-\hat \btheta_{S^\prime}\big\|_2,
\end{align}
where the last inequality is due to the Lemma \ref{lemma:SE4lasso1}. Note that the inequality \eqref{eq:sens_1} holds with probability at least $1-\exp(-C_2m)$. We will show in next that how this high probability can be absorbed into the definition of $(\epsilon,\delta)$-DP. Let us define the event $E$: inequality \eqref{eq:sens_1} holds, and we have $\PP[\bar E]\leq \delta_2$, where $\delta_2=\exp(-C_2m)$. As long as we have $m\geq C_3\log(2/\delta)$, we can get $\delta_2\leq\delta/2$. Given the event $E$ holds, we can proceed to derive the privacy guarantee of our method as follows.

For two adjacent training sets $S$ and $S^\prime$, we define $\btheta^{\min}_S$ and $\btheta^{\min}_{S^\prime}$ as follows
\begin{align*}
    \btheta^{\min}_S=\argmin_{\btheta\in\RR^d}\bar L_S(\btheta):=L_S(\btheta)+\frac{\lambda}{2}\|\btheta\|_2^2\quad \text{subject to}\quad \|\btheta\|_0 \leq s,\\
     \btheta^{\min}_{ S^\prime}=\argmin_{\btheta\in\RR^d}\bar L_{S^{\prime}}(\btheta) :=L_{S^{\prime}}(\btheta)+\frac{\lambda}{2}\|\btheta\|_2^2\quad \text{subject to}\quad \|\btheta\|_0 \leq s.
\end{align*}
Therefore, we can obtain
\begin{align}\label{eq:sens_2}
    \big\|\hat \btheta_S-\hat \btheta_{S^\prime}\big\|_2&\leq \big\|\hat \btheta_S- \btheta^{\min}_S\big\|_2+\big\|\hat \btheta_{S^\prime}- \btheta^{\min}_{S^\prime}\big\|_2+\big\|\btheta^{\min}_S-\btheta^{\min}_{S^\prime}\big\|_2\nonumber\\
    &\leq \varrho^T\big\|\btheta^{\min}_S\big\|_2+\varrho^T\big\|\btheta^{\min}_{S^\prime}\big\|_2+\big\|\btheta^{\min}_S-\btheta^{\min}_{S^\prime}\big\|_2,
\end{align}
where $\varrho<1$ and the last inequality is due to the convergence guarantee \cite{yuan2014gradient} of IGHT  for $\bar L_S, \bar L_{S^\prime}$. Since $\bar L_S$ is strongly convex with parameter $\lambda$, we have
\begin{align*}
    \la\nabla \bar L_S(\btheta^{\min}_S)-\nabla \bar L_S(\btheta^{\min}_{S^\prime}),\btheta^{\min}_S-\btheta^{\min}_{S^\prime}\ra\geq \lambda\big\|\btheta^{\min}_S-\btheta^{\min}_{S^\prime}\big\|_2^2.
\end{align*}
In addition, we have $\la \nabla \bar L_S(\btheta^{\min}_S),\btheta^{\min}_{S^\prime}-\btheta^{\min}_S\ra\geq 0$, $\la\nabla \bar L_{S^\prime}(\btheta^{\min}_{S^\prime}),\btheta^{\min}_S-\btheta^{\min}_{S^\prime}\ra\geq 0$, which implies
\begin{align*}
     \la\nabla \bar L_{S^\prime}(\btheta^{\min}_{S^\prime})-\nabla \bar L_S(\btheta^{\min}_{S^\prime}),\btheta^{\min}_S-\btheta^{\min}_{S^\prime}\ra\geq \lambda\big\|\btheta^{\min}_S-\btheta^{\min}_{S^\prime}\big\|_2^2.
\end{align*}
Thus we can obtain
\begin{align}\label{eq:gamma_bound}
    \lambda\big\|\btheta^{\min}_S-\btheta^{\min}_{S^\prime}\big\|_2
    &\leq \sqrt{2s}\big\|\nabla \bar L_{S^\prime}(\btheta^{\min}_{S^\prime})-\nabla \bar L_S(\btheta^{\min}_{S^\prime})\big\|_\infty\nonumber \\
    &= \frac{\sqrt{2s}}{2n}\big\|\nabla \ell(\btheta^{\min}_{S^\prime};\xb_i)-\nabla \ell(\btheta^{\min}_{S^\prime};\xb_{i^\prime})\big\|_\infty.
\end{align}
Since we have $\|\nabla \ell(\btheta_{S^\prime}^{\min};\xb_i)\|_\infty\leq \gamma$ for all $\xb_i$, we can get
\begin{align}\label{eq:sens_3}
    \|\btheta^{\min}_S-\btheta^{\min}_{S^\prime}\|_2&\leq \frac{\sqrt{2s}\gamma}{n\lambda}.
\end{align}
As a result, combining \eqref{eq:sens_1}, \eqref{eq:sens_2}, and \eqref{eq:sens_3}, for large enough $T$, we can obtain
\begin{align}\label{eq:final_sens}
    \Delta(\qb)\leq 2\sqrt{ms\tilde \beta}\frac{\gamma}{n\lambda}.
\end{align}
As a result, according to Lemma \ref{lemma:GaussianM}, to ensure $(\epsilon,\delta/2)$-DP, we need to add the zero mean Gaussian vector with the variance parameter 
\begin{align}\label{eq:variance}
    \sigma^2=\frac{8m \tilde \beta s\gamma^2}{n^2\epsilon^2\lambda^2}\log(2.5/\delta).
\end{align}
We use $\cM$ to denote our mechanism, i.e., Algorithm \ref{alg:DPGST2}. Given $E$ happens, $\cM$ satisfies $(\epsilon,\delta/2)$-DP. Now, we are ready show that $\cM$ satisfies $(\epsilon,\delta)$-DP. According to Remark 3.1.2 in \cite{dwork2006calibrating}, we need to prove that 
\begin{align*}
    \max_{O\in\cR}\log\frac{\PP[\cM(S)\in O]-\delta}{\PP[\cM(S^\prime)\in O]}\leq \epsilon.
\end{align*}
Since we have for all $O\in\cR$
\begin{align*}
   \PP[\cM(S)\in O]&= \PP[\cM(S)\in O~|~E]\cdot\PP[E]+\PP[\cM(S)\in O~|~\bar E]\cdot\PP[\bar E]\\
   &\leq \big(e^{\epsilon}\PP[\cM(S^\prime)\in O~|~E]+\delta/2\big)\cdot\PP[E]+\delta/2\\
   &\leq e^{\epsilon}\PP[\cM(S^\prime)\in O]+\delta/2+\delta/2,
\end{align*}
where the second inequality is due to the $(\epsilon,\delta/2)$-DP of our method given inequality \eqref{eq:sens_1} holds, and the fact that $\PP[\bar E]\leq \delta/2$. Therefore, we can obtain that 
\begin{align*}
    \max_{O\in\cR}\log\frac{\PP[\cM(S)\in O]-\delta}{\PP[\cM(S^\prime)\in O]}\leq \max_{O\in\cR}\log\frac{e^{\epsilon}\PP[\cM(S^\prime)\in O]+\delta/2+\delta/2-\delta}{\PP[\cM(S^\prime)\in O]} =\epsilon,
\end{align*}
which implies Algorithm \ref{alg:DPGST2} satisfies $(\epsilon,\delta)$-DP. And the conditions we need are: $\tilde \xb_i$  are i.i.d. sub-Gaussian random vector with parameter $\alpha$, the generated sample size $m\geq \max\{C_1s\tilde \alpha^2\log d,C_3\log(2/\delta)\}$, where $C_1,C_3$ are absolute constants. 
\end{proof}

\subsection{Proof of Theorem \ref{thm:utility_sparse}}
In this subsection, we establish the utility guarantee of Algorithm \ref{alg:DPGST2}. In order to prove the utility guarantee of our method, we need the following lemmas.
% \begin{lemma}\label{lemma:SE4lasso1}
% Suppose each row of the design matrix $\tilde \Xb\in\RR^{m\times d}$ are independent isotropic sub-Gaussian random vector with sub-Gaussian parameter $\alpha$, there exist some constants $\{C_i\}_{i=1}^2$ such that for all $\vb\in\RR^{d}$ with at most $s$ nonzero entries, if $m\geq C_1s\alpha^2\log d$, with probability at least $1-\exp(-C_2m)$, we have
% \begin{align*}
%       \tilde \mu\|\vb\|_2^2 \leq\frac{\|\tilde \Xb\vb\|_2^2}{m}\leq \tilde \beta\|\vb\|_2^2,
% \end{align*}
% where $\tilde \mu=4/5$ and $\tilde \beta=6/5$.
% \end{lemma}

\begin{lemma}\label{lemma:IGHT_converge1}
Consider the sparsity constrained problem \eqref{eq:structure_opt}. Suppose that $\bar L_S$ is $\bar \beta$-smooth, and $L_S$ satisfies Condition \ref{con:err} with parameter $\varepsilon$. There exist constants $\{C_i\}_{i=1}^5$ such that if $\eta=C_1\lambda/\bar \beta^2$, $s\geq C_2\kappa^2s^*$, where $\kappa=\bar \beta/\lambda$, the output $\hat \btheta$ of Algorithm \ref{alg:IGHT} satisfies the following with probability at least $1-\rho$
\begin{align*}
    \|\btheta_T-\btheta^*\|_2^2\leq \varrho^T \|\btheta_0-\btheta^*\|_2^2+ C_4\frac{s^*}{\bar \beta^2}(\varepsilon^2+\lambda^2\|\btheta^*\|_\infty^2),
\end{align*}
where $\varrho=1-1/(7\kappa)$. If $T$ is large enough, we have $\|\btheta_T-\btheta^*\|_2^2\leq C_5s^*(\varepsilon^2+\lambda^2\|\btheta^*\|_\infty^2)/\bar \beta^2$.
\end{lemma}
The next lemma, which has been proved in \cite{loh2013regularized}, provides the statistical error of sparse linear regression, which will be used to characterize the statistical error of our newly constructed learning problem.
\begin{lemma}\label{lemma:sr4lasso}
	For a Gaussian random vector $\bepsilon\in\RR^n$ with zero mean and variance $\nu^2\Ib_n$, if each row of $\Xb\in\RR^{n\times d}$ are independent sub-Gaussian random vector with sub-Gaussian parameter $\alpha$, we have with probability at least $1-\exp(-C_6n)$
	\begin{align*}
	\bigg\|\frac{1}{n}\Xb^\top \bepsilon\bigg\|_\infty \leq C_7 \nu \alpha \sqrt{\frac{\log d}{n}},
	\end{align*} 
	where $C_6,C_7$ are absolute constants.
\end{lemma}

\begin{proof}[Proof of Theorem \ref{thm:utility_sparse}]
% \textbf{Sparsity Constraint:} 
According to Lemma \ref{lemma:IGHT_converge1}, we can obtain that 
\begin{align}\label{eq:priv_bound_1}
    \|\hat \btheta-\btheta^*\|_2^2\leq C_1\frac{s^*}{\bar \beta^2}(\varepsilon^2+\lambda^2\|\btheta^*\|_\infty^2),
\end{align}
where $C_1$ is a universal constant. According to Algorithm \ref{alg:DPGST2}, we have 
\begin{align*}
    \tilde L(\btheta)=\frac{1}{2m}\sum_{i=1}^m\big(y_i^{\mathrm{p}}-\la\btheta,\tilde \xb_i\ra\big)^2.
\end{align*}
Note that according to Lemma \ref{lemma:SE4lasso1}, $\tilde L$ satisfies Condition \ref{con:SE} with parameters $\psi_1\tilde \beta,\psi_2\tilde \beta$, where $\psi_1=4/5,\psi_2=6/5$ . In addition, according to Lemma \ref{lemma:sr4lasso}, we have $\|\nabla \tilde L(\hat \btheta)\|_{\infty}=\|\tilde \Xb^\top\bxi/n\|_\infty=\tilde \varepsilon\leq C_2\sigma\tilde \alpha\sqrt{\log d/m}$ holds with probability at least $1-\exp(-C_3m)$.
% In addition, we have $\{\nabla f(\hat \btheta;\tilde \xb_i)\}_{i=1}^m$ are independent sub-Gaussian random vectors with parameter $\alpha$. Thus according to Lemma \ref{lemma:sr4lasso}, we can obtain that $\|\nabla \tilde L(\hat \btheta)\|_\infty\leq C_2\alpha\sigma\sqrt{\log d/m}$ holds with probability at least $1-\exp(-C_3m)$. It implies that $\tilde L$ satisfies Condition \ref{con:err}. 
As a result, according to Lemma \ref{lemma:IGHT_converge1}, we can get
\begin{align}\label{eq:priv_bound_2}
    \|\btheta^{\mathrm{p}}-\hat \btheta\|_2^2&\leq C_4 \frac{s^*}{\tilde \beta^2}\tilde \varepsilon^2,
\end{align}
where $C_2,C_3,C_4$ are universal constants. 
As a result, combining \eqref{eq:priv_bound_1} and \eqref{eq:priv_bound_2}, we can obtain
\begin{align*}
    \|\btheta^{\mathrm{p}}- \btheta^*\|_2^2&\leq 2 \|\hat \btheta^{\mathrm{p}}-\hat \btheta\|_2^2+2\|\hat \btheta-\btheta^*\|_2^2\\
    &\leq 2C_1\frac{s^*}{\bar \beta^2}(\varepsilon^2+\lambda^2\|\btheta^*\|_\infty^2)+2C_2 \frac{ s^*}{\tilde \beta^2}\tilde \varepsilon^2\\
    &\leq C_5\frac{s^*}{\bar \beta^2}(\varepsilon^2+\lambda^2\|\btheta^*\|_\infty^2)+C_6\tilde\alpha^2 \frac{ s^*}{\tilde \beta^2}\cdot\frac{\log d}{m}\sigma^2,
\end{align*}
where $C_5,C_6$ are absolute constants. Plugging the definition of $\sigma^2$ in \eqref{eq:variance}, we can get
\begin{align*}
    \|\hat \btheta^{\mathrm{p}}- \btheta^*\|_2^2 &\leq C_5\frac{s^*}{\bar \beta^2}(\varepsilon^2+\lambda^2\|\btheta^*\|_\infty^2)+C_7\tilde\alpha^2 \frac{\tilde \beta\tilde  s^{*2}}{\tilde \beta^2}\frac{  \gamma^2\log d}{n^2\epsilon^2\lambda^2}\log(2.5/\delta).
\end{align*}
Let $\lambda^2=C_8\gamma\sqrt{s^*\log d\log(1/\delta)}/(n\epsilon)$, we can get
\begin{align*}
    %  \|\hat \btheta^{\mathrm{p}}- \btheta^*\|_2^2 &\leq C_9\frac{\kappa^2}{\bar \beta^2}\varepsilon^2+\bigg(\frac{\kappa^2}{\bar\beta^2}+\frac{\tilde \alpha^2\tilde \kappa^2}{\tilde \beta}\bigg)\frac{\gamma\sqrt{s^3\log d\log(1/\delta)}}{n\epsilon}
         \|\hat \btheta^{\mathrm{p}}- \btheta^*\|_2^2 &\leq C_9\frac{s^*}{\bar \beta^2}\varepsilon^2+C_{10}\bigg(\frac{1}{\bar\beta^2}+\frac{\tilde \alpha^2}{\tilde \beta}\bigg)\frac{\gamma\sqrt{s^{*3}\log d\log(1/\delta)}}{n\epsilon },
\end{align*}
where $C_7,C_8,C_9,C_{10}$ are absolute constants. Note that according to Lemma \ref{lemma:IGHT_converge1}, Algorithm \ref{alg:DPGST2} has a linear convergence rate.
\end{proof}

\subsection{Proof of Corollary \ref{thm:improve}}
In this subsection we show that if $L_S$ further satisfies Condition \ref{con:SE}, our method can achieve an improved utility guarantee.
\begin{proof}[Proof of Corollary \ref{thm:improve}]
We first prove the privacy guarantee of our method. The proof is similar to the proof of Theorem \ref{thm:privacy_sparse}. Since we have that $L$ satisfies Condition \ref{con:SE} with parameters $\mu,\beta$, we can get the sensitivity of our method according to \eqref{eq:final_sens} as follows
\begin{align*}
        \Delta(\qb)\leq 2\sqrt{ms\tilde \beta}\frac{\gamma}{n\mu}.
\end{align*}
Therefore, according to \eqref{eq:variance}, if we add the noise with the following variance
\begin{align*}
    \sigma^2=\frac{8m \tilde \beta s\gamma^2}{n^2\epsilon^2\mu^2}\log(2.5/\delta),
\end{align*}
we can ensure that Algorithm \ref{alg:DPGST2} satisfies $(\epsilon,\delta)$-DP.

Next, we establish the utility guarantee of our method. According to \eqref{eq:priv_bound_1}, we have 
\begin{align}\label{eq:improve_bound_1}
    \|\hat \btheta-\btheta^*\|_2^2\leq C_1\frac{s^*}{ \beta^2}\varepsilon^2.
\end{align}
In addition, according to \eqref{eq:priv_bound_2}, we have 
\begin{align}\label{eq:improve_bound_2}
    \|\hat \btheta^{\mathrm{p}}-\hat \btheta\|_2^2&\leq C_2 \tilde\alpha^2 \frac{\tilde \beta\tilde  s^{*2}}{\tilde \beta^2}\frac{  \gamma^2\log d}{n^2\epsilon^2\mu^2}\log(2.5/\delta).
\end{align}
Combining \eqref{eq:improve_bound_1} and \eqref{eq:improve_bound_2}, we can get
\begin{align*}
    \|\hat \btheta^{\mathrm{p}}-\btheta^*\|_2^2\leq C_3\frac{s^*}{ \beta^2}\varepsilon^2+C_4\tilde\alpha^2 \frac{  s^{*2}\gamma^2\log d}{\tilde \beta n^2\epsilon^2\mu^2}\log(2.5/\delta),
\end{align*}
where $C_1,C_2,C_3,C_4$ are absolute constants. This completes the proof.
\end{proof}

\section{Proofs of Specific Examples}
In this section, we only establish the utility guarantees of our proposed method for different problems, including sparse linear regression and sparse logistic regression since the privacy guarantee of Algorithm \ref{alg:DPGST2} has been proved in Theorem \ref{thm:privacy_sparse}. For the ease of presentation, we use $L$ to denote $L_S$ in the following discussion.
\subsection{Proof of Corollary \ref{corollary:lasso}}
In order to prove Corollary \ref{corollary:lasso}, we only need to verify Condition \ref{con:err} for $L$, the upper bound $\gamma$ of $\ell_i$.

\begin{proof}[Proof of Corollary \ref{corollary:lasso}]
According to the objective function in \eqref{eq:linear regression}, we have the following close form of  gradient and Hessian for $L$
\begin{align*}
    \nabla L(\btheta)=\frac{1}{n}\sum_{i=1}^n(\xb_i^\top\btheta-y_i)\xb_i,\quad
    \nabla^2 L(\btheta)=\frac{\Xb^\top\Xb}{n},
\end{align*}
where $\xb_i$ is the $i$-th row of the design matrix $\Xb$. First, we verify that $\bar L$ is $\bar\beta$-smooth. According to the proof of Lemma \ref{lemma:IGHT_converge1}, we only need to show the upper bound of $\nabla^2 L(\btheta)$ restricted to some $3s$ sparse support $\Omega$. As a result, we have $\big\|\big(\nabla^2 L(\btheta)\big)_{\Omega,\Omega}\big\|_2\leq 3sK^2$, which implies that $\bar \beta=3sK^2+\lambda$. In addition, we have $\nabla L(\btheta^*)=\Xb^\top\bvarepsilon/n$. According to the proof of Corollary 2 in \cite{loh2013regularized}, we have $\|\nabla L(\btheta^*)\|_\infty\leq C_1\nu K\sqrt{\log d/n}$ holds with probability at least $1-\exp(-C_2n)$, where $C_1,C_2$ are absolute constants. Thus we have Condition \ref{con:err} holds for $L$. 
Next, we are going to estimate the parameter $\gamma$ for our utility guarantee. For the loss function on each training example, we have $\ell_i(\btheta)=(\la\xb_i,\btheta\ra-y_i)^2/2$, which implies $\nabla \ell_i(\btheta)=(\la\xb_i,\btheta\ra-y_i)\xb_i$. According to \eqref{eq:gamma_bound}, we need to verify $\|\nabla \ell_i( \btheta_{\min})\|_\infty \leq \gamma$, where $\btheta_{\min}$ is the minimizer of \eqref{eq:structure_opt}. Since we have $\|\nabla \ell_i( \btheta_{\min})\|_\infty=\|(\la\xb_i,\btheta_{\min}\ra-y_i)\xb_i\|_{\infty}\leq C_3\sqrt{s}K^2$, which implies that $\gamma \leq C_3\sqrt{s}K^2$.
Finally, plugging these results into Theorem \ref{thm:utility_sparse}, we have if $\lambda^2=C_4K^2s^{*}\sqrt{\log d\log(1/\delta)}/\big(n\epsilon\big)$,
we can get
\begin{align*}
  \|\hat\btheta^{\mathrm{p}}-\btheta^*\|_2^2
    \leq C_{5}\nu^2K^2\frac{s\log d}{n}+C_{6}\tilde\alpha^2 \frac{K^2s^{*2}\sqrt{\log d \log(1/\delta)}}{\tilde \beta n\epsilon}.
\end{align*}

\end{proof}

\subsection{Proof of Corollary \ref{corollary:logistic}}
In this subsection, we prove the results for sparse logistic regression, and we only need to verify Conditions \ref{con:err} for $L$, the upper bound $\gamma$ of $\ell_i$.
\begin{proof}[Proof of Corollary \ref{corollary:logistic}]
According to the loss function in \eqref{eq:constrained GLM}, we can obtain
\begin{align*}
    \nabla L(\btheta)=-\frac{1}{n}\sum_{i=1}^n\big(y_i-\psi(\btheta^\top\xb_i)\big)\xb_i,\quad \nabla^2 L(\btheta)=\frac{1}{n}\sum_{i=1}^n\psi^\prime(\btheta^\top\xb_i)\xb_i\xb_i^\top,
\end{align*}
  where $\psi(x)=\exp(x)/(1+\exp(x))$ and $\psi^\prime(x)=\exp(x)/(1+\exp(x))^2$. Since we have $\psi^\prime(x)\leq 1$, following the same proof procedure as before, we can get $\bar L$ is $\bar \beta$-smooth with $\bar \beta=3sK+\lambda$.
%   In addition, for all $\vb\in \RR^d$, we have
%   \begin{align*}
%   \vb^\top\nabla^2 L(\btheta)\vb=\frac{1}{n}\sum_{i=1}^n\psi^\prime(\btheta^\top\xb_i)\vb^\top\xb_i\xb_i^\top\vb.
%   \end{align*}
%   Since we have $\btheta^\top\xb_i$ is bounded, we have $\psi^\prime (x)$ is upper and lower bounded by some constants $C_1, C_2$. Therefore, according to Lemma \ref{lemma:SE4lasso1}, for all $\vb\in \RR^d$ with at most $C_2s$ nonzero entries, as long as $n\geq C_3s\alpha^2\log d$, we have with probability at least $1-\exp(-C_4n)$,
% \begin{align*}
%     \mu\|\vb\|_2^2\leq\vb^\top\nabla^2 L(\btheta)\vb\leq L\|\vb\|_2^2,
% \end{align*}
% where $\mu=4/5C_1, L=6/5C_2$, and $\{C_i\}_{i=1}^4$ are absolute constants, which implies the Condition \ref{con:SE}.
In addition,
we have  $\nabla L(\btheta^*)=\frac{1}{n}\sum_{i=1}^nb_i\xb_i$, where $b_i=y_i-\psi(\btheta^{*\top}\xb_i)$. Thus, according to the proof of Corollary 2 in \cite{loh2013regularized}, we have $\|\nabla L(\btheta^*)\|_\infty\leq C_1K\sqrt{\log d/n}$ holds with probability at least $1-C_2/d$, where $C_1, C_2$ are absolute constants. In addition, we have
\begin{align*}
     \|\nabla \ell_i(\btheta_{\min})\|_\infty=\big\|\big(y_i-\psi(\btheta_{\min}^\top\xb_i)\big)\xb_i\big\|_\infty\leq K,
\end{align*}
where the inequality is due the the fact that $y_i\in \{0,1\}$, $\psi(x)\in(0,1)$, and $\|\xb_i\|_\infty\leq K$. Thus we have $\gamma=K$ for sparse logistic regression. 

% Next, for loss function $\tilde L$, according to the same proof in the proof of Corollary \ref{corollary:lasso}, we have if $n\geq C_3s\tilde \alpha^2\log d$, $\tilde L$ satisfies Condition \ref{con:SE} with parameters $4/5C_4\tilde\beta,6/5C_5\tilde \beta$. 
% Moreover, $\|\nabla \tilde L(\hat \btheta)\|_{\infty}=\|\tilde \Xb^\top\bxi/n\|_\infty\leq C_8\sigma\tilde \alpha\sqrt{\log d/n}$, which gives us Condition \ref{con:err_tile} for $\tilde L$. 
Finally, plugging these results into Theorem \ref{thm:utility_sparse}, we have if $\lambda^2=C_6K\sqrt{s^{*}\log d\log(1/\delta)}/\big(n\epsilon\big)$,
we can get
\begin{align*}
  \|\hat\btheta^{\mathrm{p}}-\btheta^*\|_2^2
    \leq C_{7}\nu^2K^2\frac{s\log d}{n}+C_{8}\tilde\alpha^2 \frac{K\sqrt{s^{*3}\log d \log(1/\delta)}}{\tilde \beta n\epsilon}.
\end{align*}

% In addition, we have  $\nabla F(\btheta^*)=\frac{1}{n}\sum_{i=1}^nb_i\xb_i$, where $b_i=y_i-\psi(\btheta^{*\top}\xb_i)$. Therefore, according to Corollary 2 in \cite{loh2013regularized}, we have $\|\nabla F(\btheta^*)\|_\infty\leq C_5\alpha\sqrt{\log d/n}$ holds with probability at least $1-C_6/d$, where $C_5, C_6$ are absolute constants. Therefore, plugging these parameters into Corollary \ref{coro:dp}, we can obtain
% \begin{align*}
%       \EE \|\btheta_T-\btheta^*\|_2^2\leq  C_7\alpha^2 \frac{s\log d}{n}+C_8K^2\frac{s\log d }{n^2\epsilon^2}\log(1/\delta),
% \end{align*}
% where $C_7, C_8$ are absolute constants.
\end{proof}

\subsection{Proof of Corollary \ref{corollary:lasso1}}
To prove this result, we only need to verify that $L$ satisfies the sparse eigenvalue condition since other conditions has been previously verified in the proof of Corollary \ref{corollary:lasso}.
\begin{proof}[Proof of Corollary \ref{corollary:lasso1}]
Since we have $\nabla^2L(\btheta)=\Xb^\top\Xb/n$, according to Proposition 1 in \citet{agarwal2010fast}, we can obtain that $L$ satisfies Condition \ref{con:SE} with parameters $\beta=6/5$ and $\mu=4/5$ with probability at least $1-\exp(-C_1n)$ if we have $n\geq C_2s\log d$, where $C_1,C_2$ are absolute constants. Therefore, following the same proof procedure as in the proof of Theorem \ref{thm:privacy_sparse}, this high probability can be absorbed into the $\delta$ term in the $(\epsilon,\delta)$-DP. As a results, we complete the proof. 
\end{proof}

\subsection{Proof of Corollary \ref{corollary:logistic1}}
To prove this result, we only need to verify that $L$ satisfies the sparse eigenvalue condition since other conditions has been previously verified in the proof of Corollary \ref{corollary:logistic}.
\begin{proof}[Proof of Corollary \ref{corollary:lasso1}]
Since we have $\nabla^2L(\btheta)=(n)^{-1}\sum_{i=1}^n\psi^\prime(\btheta^\top\xb_i)\xb_i\xb_i^\top$, and $\psi^\prime(\btheta^\top\xb_i$ is upper and lower bounded by some constants $C_1,C_2$, we can follow the same procedure as in the proof of  Corollary \ref{corollary:lasso1} to show that $L$ satisfies Condition \ref{con:SE} with parameters $\beta=6/5C_1$ and $\mu=4/5C_2$. As a results, we complete the proof. 
\end{proof}

\section{Proofs of Additional Lemmas}
In this section, we prove the additional lemmas used in the proofs of the main results. For the ease of presentation, we use $L$ to denote $L_S$.
\subsection{Proof of Lemma \ref{lemma:IGHT_converge1}}
\begin{proof}
According to Algorithm \ref{alg:IGHT}, we have
\begin{align*}
  \btheta_{t+1}=\cH_{ s}\big(\btheta_{t}-\eta\nabla \bar L(\btheta_{t})\big).
\end{align*}
 We denote $\Omega=\supp(\btheta_{t})\cup\supp(\btheta_{t+1})\cup\supp(\btheta^*)$, and we have $ s\leq|\Omega|\leq (2s+s^*)$. In addition, we denote $\tilde \btheta_{t+1}$ by $\cP_{\Omega}\big(\btheta_t-\eta\nabla \bar L(\btheta_{t})\big)$, thus we have $\btheta_{t+1}=\cH_{ s}(\tilde \btheta_{t+1})$. Furthermore, we have the following
\begin{align*}%\label{eq:optimization}
    \|\tilde \btheta_{t+1}-\btheta^*\|_2^2&=\big\|\cP_{\Omega}\big(\btheta_t-\eta\nabla \bar L(\btheta_{t})\big)-\btheta^*\big\|_2^2\nonumber\\
    &=\big\|\btheta_t-\btheta^*-\eta\cP_{\Omega}\big(\nabla \bar L(\btheta_{t})\big)\big\|_2^2\nonumber\\
    &=\Big\|\btheta_t-\btheta^*-\eta\cP_{\Omega}\big(\nabla \bar L(\btheta^*)+\big(\Hb(\gamma)\big)_{*\Omega}(\btheta_t-\btheta^*)\big)\Big\|_2^2,
\end{align*}
where the last equation is due to the fundamental theorem of calculus, $\Hb(\gamma)=\int_0^1\nabla^2 \bar L(\btheta^*+\gamma(\btheta-\btheta^*))d\gamma$, and $\Hb(\gamma)_{*\Omega}$ denotes that we restrict columns of $\Hb(\gamma)$ to the support $\Omega$. Therefore, according to the definition of $\cP_{\Omega}$, we can further obtain
\begin{align*}
 \|\tilde \btheta_{t+1}-\btheta^*\|_2^2&=\big\|\Ab(\btheta_t-\btheta^*)-\eta\cP_{\Omega}\big(\nabla \bar L(\btheta^*)\big)\big\|_2^2\\
 &\leq \|\Ab\|_2^2\cdot\|\btheta_t-\btheta^*\|_2^2+\eta^2\big\|\cP_{\Omega}\big(\nabla \bar L(\btheta^*)\big)\big\|_2^2 -2\eta\la\Ab(\btheta_t-\btheta^*),\cP_{\Omega}\big(\nabla \bar L(\btheta^*)\big)\ra,
\end{align*}
where we have $\Ab=\Ib-\eta\big(\Hb(\gamma)\big)_{\Omega\Omega}$. Thus by Young's inequality, we can obtain
\begin{align*}
   -2\eta\la\Ab(\btheta_t-\btheta^*),\cP_{\Omega}\big(\nabla \bar L(\btheta^*)\big)\ra\leq\frac{2\eta\bar\beta}{7}\|\btheta_t-\btheta^*\|_2^2 +\frac{14\eta}{\bar\beta}\big(\|\Ab\|_2^2\cdot\big\|\cP_{\Omega}\big(\nabla \bar L(\btheta^*)\big)\big\|_2^2\big).
\end{align*}
Therefore, we can get
\begin{align*}
\|\tilde \btheta_{t+1}-\btheta^*\|_2^2
 &\leq \|\Ab\|_2^2\cdot\|\btheta_t-\btheta^*\|_2^2+\eta^2\big\|\cP_{\Omega}\big(\nabla \bar L(\btheta^*)\big)\big\|_2^2 +\frac{2\eta\bar\beta}{7}\|\btheta_t-\btheta^*\|_2^2+\frac{14\eta}{\bar\beta}\big(\|\Ab\|_2^2\cdot\big\|\cP_{\Omega}\big(\nabla \bar L(\btheta^*)\big)\big\|_2^2\big)\\
  &\leq \Big(1-\frac{5\eta\bar\beta}{7}\Big)\|\btheta_t-\btheta^*\|_2^2+\Big(\frac{14\eta}{\bar\beta}-14\eta^2\Big)\big\|\cP_{\Omega}\big(\nabla \bar L(\btheta^*)\big)\big\|_2^2,
\end{align*}
where the last inequality is due to the Condition \ref{con:SE}.

In addition, according to Lemma 3.3 in \cite{li2016nonconvex}, we have
\begin{align}\label{eq:ht_contract}
    \|\btheta_{t+1}-\btheta^*\|_2^2\leq \bigg(1+\frac{2\sqrt{s^*}}{\sqrt{s-s^*}}\bigg)\|\tilde \btheta_{t+1}-\btheta^*\|_2^2,
\end{align}
which implies that
\begin{align*}
 \|\btheta_{t+1}-\btheta^*\|_2^2
 &\leq \alpha\Big(1-\frac{5\eta\bar\beta}{7}\Big)\|\btheta_t-\btheta^*\|_2^2+\alpha(2s+s^*)\Big[\Big(\frac{14\eta}{\bar\beta}-14\eta^2\Big)\big(\|\nabla \bar L(\btheta^*)\|_{\infty}^2\big)\Big],
\end{align*}
where $\alpha=1+2\sqrt{s^*}/\sqrt{s-s^*}$. Since we have $\eta=2\lambda/\bar \beta^2$, as long as $s\geq (4\kappa^2+1)s^*$, where $\kappa=\bar \beta/ \lambda$, we can get
\begin{align*}
 \| \btheta_{t+1}-\btheta^*\|_2^2
 &\leq \varrho\|\btheta_t-\btheta^*\|_2^2+C_1\frac{s^*\lambda}{\bar\beta^3}\|\nabla \bar L(\btheta^*)\|_{\infty}^2,
\end{align*}
where the we have $\varrho\leq1-1/(7\kappa)<1$.

In addition, we have $\nabla \bar L(\btheta^*)=\nabla L(\btheta^*)+\lambda\btheta^*$. According to Condition \ref{con:err}, we have
\begin{align*}
  \|\nabla \bar L(\btheta^*)\|_{\infty}=\|\nabla L(\btheta^*)+\lambda\btheta^*\|_\infty\leq  \|\nabla L(\btheta^*)\|_\infty+\lambda\|\btheta^*\|_\infty\leq \varepsilon +\lambda\|\btheta^*\|_\infty.
\end{align*}
As long as we choose $\lambda=O(\varepsilon/\|\btheta^*\|_\infty)$, we can get
\begin{align}\label{eq:Expcontraction}
 \| \btheta_{t+1}-\btheta^*\|_2^2
 &\leq \varrho\|\btheta_t-\btheta^*\|_2^2+C_2\frac{s^*\lambda}{\bar\beta^3}(\varepsilon^2+\lambda^2\|\btheta^*\|_\infty^2).
\end{align}

 Thus taking sum of \eqref{eq:Expcontraction} over $t=0,1,\ldots,T-1$, we can get
\begin{align}\label{eq:contraction}
 \|\btheta_{T}-\btheta^*\|_2^2
 &\leq \varrho^T\|\btheta^*\|_2^2+C_2\frac{s^*\lambda}{\bar\beta^3(1-\varrho)}(\varepsilon^2+\lambda^2\|\btheta^*\|_\infty^2)\nonumber\\
 &\leq \varrho^T\|\btheta^*\|_2^2+C_2\frac{s^*}{\bar\beta^2}(\varepsilon^2+\lambda^2\|\btheta^*\|_\infty^2).
\end{align}
Therefore, if we have 
\begin{align*}
    T\geq C_3\kappa \log\frac{ \bar \beta\|\btheta^*-\btheta_0\|_2}{s^* (\varepsilon+\lambda\|\btheta^*\|_{\infty})},
\end{align*}
we can obtain that
\begin{align*}
    \|\btheta_{T}-\btheta^*\|_2^2\leq C_4\frac{s^*}{\bar \beta^2}(\varepsilon^2+\lambda^2\|\btheta^*\|_\infty^2),
\end{align*}
where $\{C_i\}_{i=1}^4$ are universal constants.
\end{proof}

\bibliographystyle{ims}
\bibliography{arxiv_bib}

\end{document}